\documentclass[11pt]{article}

\usepackage[margin=1in]{geometry}
\usepackage[T1]{fontenc}
\usepackage[utf8]{inputenc}
\usepackage{lmodern}
\usepackage[hyphens]{url}
\usepackage{graphicx}
\usepackage{natbib}
\usepackage{subfigure}
\usepackage{booktabs}
\usepackage{amsfonts}
\usepackage{nicefrac}
\usepackage{microtype}
\usepackage{xcolor}
\usepackage{float}
\usepackage{hyperref}
\hypersetup{
    colorlinks=true,
    citecolor=blue,
    linkcolor=red,
    urlcolor=blue
}
\usepackage{titling}
\thanksmarkseries{arabic}

\usepackage{amsmath}
\usepackage{amssymb,mathtools,amsthm}
\usepackage{bm}
\usepackage{algorithm}
\usepackage{algpseudocode}

\usepackage{amsmath,amsfonts,bm}

\def\eqref#1{equation~\ref{#1}}

\def\1{\bm{1}}

\DeclareMathAlphabet{\mathsfit}{\encodingdefault}{\sfdefault}{m}{sl}
\SetMathAlphabet{\mathsfit}{bold}{\encodingdefault}{\sfdefault}{bx}{n}

\newtheorem{thm}{Theorem}
\newtheorem{lem}{Lemma}
\newtheorem{theorem}[thm]{Theorem}
\newtheorem{lemma}[lem]{Lemma}
\newtheorem{corollary}{Corollary}
\newtheorem{assumption}{Assumption}

\title{Few-for-Many Best-Arm Identification in Multi-Objective Bandits}

\author{
  Bo Xue\thanks{City University of Hong Kong.
  Email: \texttt{boxue4-c@my.cityu.edu.hk}}
  \quad
  Yilu Liu\thanks{City University of Hong Kong.
  Email: \texttt{J.Cheng@my.cityu.edu.hk}}
  \quad
  Haodong Jing\thanks{Xi'an Jiaotong University.
  Email: \texttt{jinghd@stu.xjtu.edu.cn}}
  \quad
  Hongzong Li\thanks{Northwest Polytechnical University.
  Email: \texttt{lihongzong@nwpu.edu.cn}}
  \quad
  Xi Lin\thanks{Xi'an Jiaotong University.
  Email: \texttt{xi.lin@xjtu.edu.cn}}
  \quad
  Shuang Qiu{\color{blue}${}^\dag$}\thanks{City University of Hong Kong.
  Email: \texttt{shuanqiu@cityu.edu.hk}}
}

\usepackage[symbol]{footmisc}

\begin{document}

\date{\today}
\footnotetext[2]{Corresponding Author}

\maketitle

\begin{abstract}
Real-world systems often serve many objectives but can deploy only a few options, whose performance must be learned from noisy feedback. Modeling each option as an arm in a multi-objective bandit, we introduce \emph{few-for-many best-arm identification}, which seeks a small set of arms minimizing the worst weighted degradation from the objective-wise optima, rather than the best arm for every objective. We consider both known objective groups and latent groups specified only by a representative budget, and propose \textsc{GFM-Elim} and \textsc{AFM-Elim} with instance-dependent sample-complexity guarantees. For Gaussian rewards, we further derive a lower bound for the latent-group setting and show that it is asymptotically attainable with high probability by an oracle-based Track-and-Stop procedure. Our results show that sharing observations across objectives and ignoring irrelevant objective-wise distinctions can substantially reduce the sampling cost, as confirmed by experiments on synthetic and real-world tasks.
\end{abstract}

\section{Introduction}
Modern decision systems frequently need to serve many evaluation criteria or user groups with only a handful of deployable choices. This motivates the \emph{few-for-many} paradigm, which selects a small set of complementary solutions that jointly cover many objectives \citep{Liu:2024MaCP,Lin:2025few,Liu:2025F4M,Shang:2026F4M}. For example, clinical trials may seek a few treatment options that collectively perform well across multiple efficacy and safety outcomes \citep{Kavelaars:2020,Villar:2015}; online advertising platforms may use a few campaigns to reach heterogeneous target audiences \citep{Geng:2020,Schwartz:2017}; and LLM services may route diverse queries among a limited pool of models with different capabilities and inference costs \citep{Chen:2024FrugalGPT,Ong:2025RouteLLM}.

A key challenge in these applications is that the multi-objective performance of each candidate is unknown and must be learned from noisy observations, such as patient outcomes, user responses to advertisements, or repeated model evaluations. Few-for-many selection thus becomes a statistical identification problem: the learner must find a small representative set whose members jointly serve many objectives, while simultaneously learning their unknown performance. Multi-objective multi-armed bandits offer a natural framework for this setting, associating each arm with an unknown vector-valued mean reward \citep{Drugan:2013,bandit_algorithm:2020}. In the pure-exploration regime, the goal is to identify the desired representative set with probability at least $1-\delta$ using as few arm pulls as possible \citep{Audibert:2010,Garivier:2016}.

Existing multi-objective formulations typically lie at one of two extremes. At one end, Pareto-front identification \citep{Auer:2016,Ararat:2023,Crepon:2024} seeks all nondominated arms; as the number of objectives grows, however, the Pareto set can become large, increasing both the statistical cost of recovering it and the burden of choosing among its members. At the other end, scalarization-based methods \citep{Paria:2020,Lin:2024SmoothTchebycheff} collapse multiple objectives into a scalar utility and return a single solution for a given preference, which can be overly restrictive when different objectives favor different arms. These limitations call for an intermediate goal: identifying a few complementary arms that jointly serve many objectives.


The closest bandit formulation, \emph{multi-objective best-arm identification} (MO-BAI), seeks the best arm for every objective \citep{Chen:2025MOBAI}. This target is appropriate when a separate decision can be deployed for each objective, but it can be unnecessarily costly when many objectives must be served under a limited deployment budget. The few-for-many principle instead selects a prescribed number of complementary solutions such that every objective is well served by at least one of them \citep{Lin:2025few,Liu:2025F4M}. However, existing studies apply this principle to deterministic many-objective optimization, whereas in the bandit setting both the objective-wise ideal values and the quality of any candidate representative set must be inferred from noisy observations.

To bridge this gap, we formulate a stochastic counterpart of few-for-many optimization, which we call \emph{few-for-many best-arm identification}. The quality of a representative set is measured by its worst weighted degradation from the objective-wise ideal values. We consider two forms of prior structure. In the \emph{known-group} setting, the objectives are partitioned in advance and one arm is selected for each group. In the \emph{latent-group} setting, only a representative budget $L$ is specified, and the selected arms implicitly induce groups according to which representative best serves each objective. In the latter setting, $L=1$ recovers a single robust representative, whereas a sufficiently large $L$ recovers objective-wise identification; the intermediate regime $1<L\ll M$ captures the practically relevant case in which a few deployable arms must jointly serve many objectives. Our main contributions are as follows:
\begin{itemize}
    \item We introduce a few-for-many bandit framework that covers both known and latent objective groups. It bridges MO-BAI \citep{Chen:2025MOBAI} and few-for-many optimization \citep{Lin:2025few}, accommodating both noisy observations and a limited representative budget.

    \item We develop two phase-based elimination algorithms, \textsc{GFM-Elim} for known objective groups and \textsc{AFM-Elim} for latent groups, and establish finite-time sample-complexity guarantees for both. Unlike MO-BAI \citep{Chen:2025MOBAI}, our algorithms need not resolve objective-wise distinctions that do not affect the optimal representatives, which can substantially reduce the sampling cost in many-objective settings.

    \item For Gaussian rewards, we establish an information-theoretic lower bound for the latent-group setting that characterizes its fundamental instance-dependent sample complexity. We further show in Appendix~\ref{app:afm-tas} that this lower bound is asymptotically attainable with high probability by an oracle-based Track-and-Stop procedure.

    \item We evaluate the proposed algorithms on synthetic and real-world tasks. The results show that \textsc{GFM-Elim} and \textsc{AFM-Elim} reliably identify compact representative sets while requiring fewer samples than existing multi-objective identification baselines.
\end{itemize}

\paragraph{Related Work.}
Best-arm identification (BAI) aims to identify the optimal decision using as few samples as possible. \citet{Garivier:2016} established the instance-dependent asymptotic lower bound for classical BAI and introduced Track-and-Stop, which has since been extended to structured pure-exploration problems, including combinatorial answer spaces and multiple correct answers \citep{Chen:2014CPE,Degenne:2019,Jourdan:2021}. For vector-valued rewards, multi-objective bandits have been studied under Pareto-based and preference-based criteria \citep{Drugan:2013,Auer:2016,Shukla:2024PrePEx,xue2026costaware,xue2026efficientmatrix,xue2026safe}. In particular, \citet{Chen:2025MOBAI} identify the best arm for every objective and develop asymptotically optimal sampling strategies. In parallel, few-for-many optimization seeks a small set of complementary solutions that jointly serve many objectives, as exemplified by the many-objective cover problem \citep{Liu:2024MaCP}, Tchebycheff set scalarization \citep{Lin:2025few}, and the general F4M framework \citep{Liu:2025F4M,Shang:2026F4M}. These works mainly consider deterministic optimization, where objective evaluations or gradients are directly accessible. Our work connects these two lines of research by studying few-for-many identification with unknown vector-valued means, where only the distinctions that can affect the optimal representatives need to be resolved from noisy sequential observations. A more detailed discussion is provided in Appendix~\ref{app:extended-related-work}.

\section{Preliminaries}
For any positive integer $n$, let $[n]=\{1,\ldots,n\}$. For a finite set $\mathcal X$, $|\mathcal X|$ denotes its cardinality. For any $x\in\mathbb R$, let $[x]_+=\max\{x,0\}$. We use $\Delta_K=\{w\in\mathbb R_+^K:\sum_{a=1}^K w_a=1\}$ to denote the probability simplex over $K$ arms. We consider a stochastic multi-armed bandit problem with $K$ arms and $M$ objectives. At each time step $t$, the learner selects an arm $a_t\in[K]$ and observes a vector-valued reward
$
\bm r_t=(r_{t,1},\ldots,r_{t,M})^\top\in\mathbb R^M.
$
For each arm $a\in[K]$, let $\boldsymbol\mu_a=(\mu_{a,1},\ldots,\mu_{a,M})^\top$ denote its unknown mean reward vector, where
\[
\mu_{a,m}=\mathbb E[r_{t,m}\mid \mathcal F_{t-1},a_t=a], \quad \forall m\in[M],
\]
and $\mathcal F_{t-1}$ denotes the history available before time step $t$. We assume that the reward noise is $\sigma^2$-sub-Gaussian; that is, for any $s\in\mathbb R$,
\begin{equation}\label{eq:subgaussian-assumption}
\mathbb E\left[
\exp\left(
s\left(r_{t,m}-\mu_{a,m}\right)
\right)
\middle|
\mathcal F_{t-1},a_t=a
\right]
\leq
\exp\left(\frac{\sigma^2s^2}{2}\right).
\end{equation}
The reward coordinates may be statistically dependent, since our finite-time analysis relies only on the marginal concentration property in Eq.~(\ref{eq:subgaussian-assumption}). For ease of presentation, we assume $\mu_{a,m}\in[0,1]$.

We further assume that each objective has a unique best arm. For every objective $m\in[M]$, we define its ideal value and best arm as
\begin{equation*}
\mu_m^\star=\max_{a\in[K]}\mu_{a,m},
\quad
b_m^\star=\arg\max_{a\in[K]}\mu_{a,m}.
\end{equation*}

\paragraph{Few-for-Many Identification.} When the number of objectives is large, identifying a separate best arm for every objective may yield many arms, whereas requiring a single arm to perform well on all objectives can be overly restrictive. Motivated by the few-for-many paradigm \citep{Lin:2025few}, we seek a small number of representative arms that together perform well across many objectives, with each representative serving a group of objectives. We consider two complementary settings, depending on whether the objective groups are known in advance.

\paragraph{Known Objective Groups.}
In the \emph{known-group} setting, objectives are partitioned in advance, and all objectives in a group share one representative arm. Let $1\leq Q\ll M$ and let $\mathcal G=\{G_1,\ldots,G_Q\}$ be a partition of $[M]$, i.e., $[M]=\cup_{q=1}^Q G_q$, where $G_q$ contains the objectives in the $q$-th group. Let $\lambda_m>0$ be a known weight for objective $m\in[M]$, which reflects its relative importance.

For every group $q\in[Q]$ and arm $a\in[K]$, we define the group coverage loss as
\begin{equation*}
R_q(a)
=
\max_{m\in G_q}
\lambda_m\left(\mu_m^\star-\mu_{a,m}\right).
\end{equation*}
Thus, $R_q(a)$ measures the largest weighted degradation incurred when arm $a$ represents all objectives in $G_q$. Accordingly, the identification target is
\begin{equation*}
    \boldsymbol a_{\mathcal G}^\star = (a_1^\star,\ldots,a_Q^\star),
    \qquad
    a_q^\star = \arg\min_{a\in[K]}R_q(a).
\end{equation*}
When $Q=1$, the problem reduces to identifying a single robust representative for all objectives. When $Q=M$ and $G_q=\{q\}$, it reduces to identifying the best arm for every objective.

\paragraph{Latent Objective Groups.}
In the \emph{latent-group} setting, the objective partition is not known in advance. Instead, the learner is given a representative budget $1\leq L\ll M$ and seeks $L$ arms that jointly provide good coverage of all objectives, with the objective groups implicitly induced by the selected representatives.

For any nonempty arm set $S\subseteq[K]$, we define its overall coverage loss as
\begin{equation*}
\Phi(S)
=
\max_{m\in[M]}
\lambda_m\left(\mu_m^\star-V_m(S)\right),
\qquad
V_m(S)=\max_{a\in S}\mu_{a,m}.
\end{equation*}
Hence, $\Phi(S)$ measures the largest weighted degradation incurred when each objective is served by its best arm in $S$ rather than by its best arm in $[K]$. Accordingly, the identification target is
\begin{equation*}
S_L^\star
=
\arg\min_{S\in\mathcal S_L}\Phi(S),
\qquad
\mathcal S_L=\left\{S\subseteq[K]:|S|=L\right\}.
\end{equation*}
The optimal set $S_L^\star$ naturally induces an assignment of objectives to representatives. Specifically, let $\pi_L^\star(m)=\arg\max_{a\in S_L^\star}\mu_{a,m}$. The objectives assigned to the same representative arm then form a latent group:
\begin{equation*}
G_a^\star
=
\left\{
m\in[M]:
\pi_L^\star(m)=a
\right\},
\quad
a\in S_L^\star.
\end{equation*}
Thus, the objective partition is induced directly by $S_L^\star$ rather than identified separately. When $L=1$, the problem reduces to identifying a single robust representative for all objectives, which coincides with the known-group setting with $Q=1$. At the other extreme, once $L\geq|\{b_m^\star:m\in[M]\}|$ (in particular, whenever $L\geq M$), the representative set can contain all objective-wise best arms and attains zero coverage loss.

\section{Finite-Time Elimination Algorithms}
This section presents finite-time elimination algorithms for the two few-for-many settings: known objective groups and latent objective groups.

\subsection{Grouped Few-for-Many Elimination}

We begin with the known-group setting. Although the objective groups are given, the group coverage losses cannot be evaluated directly because the mean rewards are unknown. To address this, we propose \textsc{GFM-Elim}, a phase-based elimination algorithm that jointly learns the mean rewards and the representative of each group. The complete procedure is summarized in Algorithm~\ref{alg:gse}.

At phase $r$, \textsc{GFM-Elim} maintains benchmark candidate sets $\mathcal B_{m,r}$ for the objective-wise ideal values and representative candidate sets $\mathcal C_{q,r}$ for the group representatives, all initialized to $[K]$. Let $\mathcal Q_r=\{q:|\mathcal C_{q,r}|>1\}$ denote the set of unresolved groups and $\mathcal M_r=\cup_{q\in\mathcal Q_r}G_q$ be the associated objectives. The algorithm samples only the arms in
$
\mathcal U_r=
\left(\cup_{q\in\mathcal Q_r}\mathcal C_{q,r}\right)
\cup
\left(\cup_{m\in\mathcal M_r}\mathcal B_{m,r}\right),
$
so that each observation is shared across all unresolved groups for which the arm remains relevant.

Each phase refines the candidates at a finer accuracy than the previous one. Specifically, let
\begin{equation}\label{eq:phase-confidence}
\epsilon_r=2^{-r},
\qquad
\delta_r=\frac{6\delta}{\pi^2r^2},
\qquad
n_r=
\left\lceil
\frac{2\sigma^2}{\epsilon_r^2}
\log\left(\frac{2KM}{\delta_r}\right)
\right\rceil.
\end{equation}
This choice of $n_r$ ensures that, with probability at least $1-\delta_r$, the empirical means of all arm--objective pairs are simultaneously accurate to within $\epsilon_r$. Each active arm is therefore sampled until it has accumulated $n_r$ observations, from which \textsc{GFM-Elim} computes the empirical means $\widehat\mu_{a,m}^{(r)}$.

\begin{algorithm}[tb]
\caption{\textsc{GFM-Elim}: Grouped Few-for-Many Elimination}
\label{alg:gse}
\begin{algorithmic}[1]
\Require Confidence level $\delta$, groups $\mathcal G=\{G_1,\ldots,G_Q\}$, and weights $\{\lambda_m\}_{m\in[M]}$
\State $\mathcal B_{m,1}\gets[K]$ for all $m$;
$\mathcal C_{q,1}\gets[K]$ for all $q$
\For{$r=1,2,\ldots$}
    \State $\mathcal Q_r\gets\{q:|\mathcal C_{q,r}|>1\}$
    \If{$\mathcal Q_r=\emptyset$} \State \Return the unique arms in
    $(\mathcal C_{q,r})_{q\in[Q]}$ \EndIf
    \State $\mathcal M_r\gets\bigcup_{q\in\mathcal Q_r}G_q$;
    $\mathcal U_r\gets(\bigcup_{q\in\mathcal Q_r}\mathcal C_{q,r})
    \cup(\bigcup_{m\in\mathcal M_r}\mathcal B_{m,r})$
    \State Sample every $a\in\mathcal U_r$ until it has $n_r$ observations and compute $\widehat\mu_{a,m}^{(r)}$
    \For{$m\in\mathcal M_r$}
        \State $\widehat\mu_{m,r}^\star=\max_{b\in\mathcal B_{m,r}}\widehat\mu_{b,m}^{(r)}$
        \State $\mathcal B_{m,r+1}\gets\{a\in\mathcal B_{m,r}:
        \widehat\mu_{a,m}^{(r)}\geq\widehat\mu_{m,r}^\star-2\epsilon_r\}$
    \EndFor
    \For{each $q\in\mathcal Q_r$}
        \State Compute $\underline R_{q,r}(a)$ and $\overline R_{q,r}(a)$ using Eq.~(\ref{eq:group-loss-estimated})
        \State $\mathcal C_{q,r+1}\gets\{a\in\mathcal C_{q,r}:
        \underline R_{q,r}(a)\leq\min_b\overline R_{q,r}(b)\}$
    \EndFor
    \State Keep all inactive sets unchanged
\EndFor
\end{algorithmic}
\end{algorithm}

\textsc{GFM-Elim} first refines the benchmark candidates. For each active objective $m\in\mathcal M_r$, the empirical benchmark is defined as $\widehat\mu_{m,r}^\star=\max_{b\in\mathcal B_{m,r}}\widehat\mu_{b,m}^{(r)}$. On the confidence event, an arm whose empirical mean falls more than $2\epsilon_r$ below this benchmark cannot be the best arm for objective $m$, and is therefore eliminated, i.e., $\mathcal B_{m,r+1}=\{a\in\mathcal B_{m,r}:\widehat\mu_{a,m}^{(r)}\geq\widehat\mu_{m,r}^\star-2\epsilon_r\}$.

The representative candidates are then updated according to their group coverage losses. Since both the empirical benchmark and the empirical mean of arm $a$ are accurate to within $\epsilon_r$, the coverage loss $R_q(a)$ is bracketed by
\begin{equation}\label{eq:group-loss-estimated}
\begin{aligned}
\underline R_{q,r}(a)
=
\max_{m\in G_q}\lambda_m
\left[\widehat\mu_{m,r}^\star-\widehat\mu_{a,m}^{(r)}-2\epsilon_r\right]_+,
\, \quad 
\overline R_{q,r}(a)
=
\max_{m\in G_q}\lambda_m
\left[\widehat\mu_{m,r}^\star-\widehat\mu_{a,m}^{(r)}+2\epsilon_r\right]_+.
\end{aligned}
\end{equation}
An arm is eliminated as a representative candidate once its optimistic loss exceeds the smallest pessimistic loss in its group, i.e., $\mathcal C_{q,r+1}=\{a\in\mathcal C_{q,r}:\underline R_{q,r}(a)\leq\min_b\overline R_{q,r}(b)\}$.

This procedure continues until every group has a unique representative candidate. We next state its finite-time guarantee, which reflects the arm-dependent stopping behavior described above. To characterize the complexity, we define the group-representative gap and the objective-wise gap as
\begin{equation*}
\Delta_{q,a}
=
R_q(a)-R_q(a_q^\star),
\quad
g_{m,a}
=
\mu_m^\star-\mu_{a,m}.
\end{equation*}
For each $q\in[Q]$, let $\overline\lambda_q=\max_{m\in G_q}\lambda_m$ and define its identification radius as $\rho_q=\min_{a\neq a_q^\star}\frac{\Delta_{q,a}}{8\overline\lambda_q}$. The effective gap of arm $a$ is then defined as
\begin{equation}\label{eq:effective-arm-gap}
\Gamma_a
=
\min_{q\in[Q]}
\left\{
\max\left\{
\rho_q,
\frac{\Delta_{q,a}}{8\overline\lambda_q}
\right\},
\min_{m\in G_q}
\max\left\{
\rho_q,
\frac{g_{m,a}}{4}
\right\}
\right\}.
\end{equation}
The two terms in Eq.~(\ref{eq:effective-arm-gap}) correspond to the two roles of arm $a$: the first determines when it can be eliminated as a representative candidate for group $q$, while the second determines when it is no longer needed to estimate the objective-wise benchmarks within that group. The radius $\rho_q$ reflects the fact that both roles terminate once group $q$ is resolved. Hence, $\Gamma_a$ captures the hardest role of arm $a$ across all groups and determines the last phase in which it is sampled.

\begin{thm}[Instance-dependent sample complexity]\label{thm:gse-sample-complexity}
With probability at least $1-\delta$, Algorithm~\ref{alg:gse} returns the optimal representative tuple $\bm a_{\mathcal G}^\star$ using at most $T_{\rm GFM}(\boldsymbol\mu,\delta)$ arm pulls, where
\begin{equation}\label{eq:gse-explicit-complexity}
\begin{aligned}
T_{\rm GFM}(\boldsymbol\mu,\delta)
&=
\widetilde O\left(
\sigma^2
\sum_{a=1}^K
\frac{1}{\Gamma_a^2}
\log\frac{KM}{\delta}
\right).
\end{aligned}
\end{equation}
\end{thm}

Theorem~\ref{thm:gse-sample-complexity} shows that \textsc{GFM-Elim} is instance-adaptive at the arm level: each arm is sampled only until all of its benchmark and representative roles have been resolved. Consequently, the complexity depends on the hardest remaining role of each arm through $\Gamma_a$, rather than accumulating a separate sampling cost for every group. This shared-feedback effect can improve the leading complexity term by a factor of up to $Q$ relative to running an independent elimination procedure for each group.

\subsection{Adaptive Few-for-Many Elimination}

We next turn to the latent-group setting, where no objective partition is available and the learner must identify the optimal representative set directly. We propose \textsc{AFM-Elim}, which lifts the elimination principle of \textsc{GFM-Elim} from individual arms to candidate sets. Importantly, although elimination is performed at the set level, sampling remains at the base-arm level, so that observations are shared across candidate sets. The complete procedure is summarized in Algorithm~\ref{alg:afm-elim}.

At phase $r$, \textsc{AFM-Elim} maintains a benchmark candidate set $\mathcal B_{m,r}$ for each objective $m\in[M]$ and a family $\mathcal C_r$ of candidate representative sets, initialized as $\mathcal B_{m,1}=[K]$ and $\mathcal C_1=\mathcal S_L$, respectively. The algorithm samples only the base arms that either belong to an active candidate set or remain relevant for estimating an objective-wise benchmark, i.e.,
$\mathcal U_r
=
\left(
\cup_{S\in\mathcal C_r}S
\right)
\cup
\left(
\cup_{m=1}^M\mathcal B_{m,r}
\right).
$
Thus, a single pull of an arm simultaneously refines the estimates of all active candidate sets containing that arm, as well as all objectives for which it remains a benchmark candidate.

\textsc{AFM-Elim} uses the same phase schedule in Eq.~(\ref{eq:phase-confidence}) and the same benchmark-refinement rule as \textsc{GFM-Elim}. The remaining candidate sets are then compared through their empirical coverage losses. For each $S\in\mathcal C_r$, its empirical value on objective $m$ is $\widehat V_{m,r}(S)=\max_{a\in S}\widehat\mu_{a,m}^{(r)}$. Since both the empirical benchmark $\widehat\mu_{m,r}^\star$ and the empirical set value $\widehat V_{m,r}(S)$ are accurate to within $\epsilon_r$ on the confidence event, the coverage loss $\Phi(S)$ is bracketed by
\begin{equation}\label{eq:optimistic-pessimistic-set-loss}
\begin{aligned}
\underline\Phi_r(S)
=
\max_{m\in[M]}\lambda_m
\left[\widehat\mu_{m,r}^\star-\widehat V_{m,r}(S)-2\epsilon_r\right]_+,
\quad
\overline\Phi_r(S)
=
\max_{m\in[M]}\lambda_m
\left[\widehat\mu_{m,r}^\star-\widehat V_{m,r}(S)+2\epsilon_r\right]_+.
\end{aligned}
\end{equation}
A candidate set is eliminated once its optimistic loss exceeds the pessimistic loss of some other active set, i.e., $\mathcal C_{r+1}=\{S\in\mathcal C_r:\underline\Phi_r(S)\leq\min_{S'\in\mathcal C_r}\overline\Phi_r(S')\}$.

\begin{algorithm}[tb]
\caption{\textsc{AFM-Elim}: Adaptive Few-for-Many Elimination}
\label{alg:afm-elim}
\begin{algorithmic}[1]
\Require Confidence level $\delta$, representative budget $L$, and weights
$\{\lambda_m\}_{m\in[M]}$
\State $\mathcal B_{m,1}\gets[K]$ for all $m$; $\mathcal C_1\gets\mathcal S_L$
\For{$r=1,2,\ldots$}
    \If{$|\mathcal C_r|=1$}
        \State \Return the unique set in $\mathcal C_r$
    \EndIf
    \State $\mathcal U_r\gets(\bigcup_{S\in\mathcal C_r}S)
    \cup(\bigcup_{m=1}^M\mathcal B_{m,r})$
    \State Sample every $a\in\mathcal U_r$ until it has $n_r$ observations and compute $\widehat\mu_{a,m}^{(r)}$
    \For{$m\in[M]$}
        \State $\widehat\mu_{m,r}^\star=\max_{b\in\mathcal B_{m,r}}\widehat\mu_{b,m}^{(r)}$
        \State $\mathcal B_{m,r+1}\gets\{a\in\mathcal B_{m,r}:
        \widehat\mu_{a,m}^{(r)}\geq\widehat\mu_{m,r}^\star-2\epsilon_r\}$
    \EndFor
    \For{$S\in\mathcal C_r$}
        \State Compute $\underline\Phi_r(S)$ and $\overline\Phi_r(S)$ using
        Eq.~(\ref{eq:optimistic-pessimistic-set-loss})
    \EndFor
    \State $\mathcal C_{r+1}\gets
    \{S\in\mathcal C_r:
    \underline\Phi_r(S)\leq
    \min_{S'\in\mathcal C_r}\overline\Phi_r(S')\}$
\EndFor
\end{algorithmic}
\end{algorithm}

This procedure continues until a unique candidate set remains. We next characterize its arm-dependent stopping behavior. For every candidate set $S\in\mathcal S_L$, we define its identification gap and the minimum set-level gap as
\begin{equation*}
\Delta_L(S)
=
\Phi(S)-\Phi(S_L^\star),
\qquad
\Delta_{L,\min}
=
\min_{S\in\mathcal S_L\setminus\{S_L^\star\}}
\Delta_L(S)>0.
\end{equation*}
Let $\overline\lambda=\max_{m\in[M]}\lambda_m$ and define the set-identification radius as
$\rho_L=\Delta_{L,\min}/(8\overline\lambda)$. The effective gap of each base arm $a\in[K]$ is defined as
\begin{equation}\label{eq:latent-effective-gap}
\Gamma_{L,a}
=
\min\left\{
\max\left\{
\rho_L,
\min_{\substack{S\in\mathcal S_L\\ a\in S}}
\frac{\Delta_L(S)}{8\overline\lambda}
\right\},
\;
\min_{m\in[M]}
\max\left\{
\rho_L,
\frac{g_{m,a}}{4}
\right\}
\right\}.
\end{equation}
The first term in Eq.~(\ref{eq:latent-effective-gap}) determines when all candidate sets containing arm $a$ can be eliminated, whereas the second determines when arm $a$ is no longer needed for benchmark estimation. The truncation at $\rho_L$ reflects the fact that both roles become irrelevant once the optimal representative set has been identified. Hence, $\Gamma_{L,a}$ determines the last phase in which arm $a$ is sampled.

\begin{theorem}[Instance-dependent sample complexity]
\label{thm:afm-sample-complexity}
With probability at least $1-\delta$, Algorithm~\ref{alg:afm-elim} returns the optimal representative set $S_L^\star$ using at most $T_{\rm AFM}(\boldsymbol\mu,\delta)$ arm pulls, where
\begin{equation}\label{eq:afm-explicit-sample-complexity}
T_{\rm AFM}(\boldsymbol\mu,\delta)
=
\widetilde O\left(
\sigma^2
\sum_{a=1}^K
\frac{1}{\Gamma_{L,a}^2}
\log\frac{KM}{\delta}
\right).
\end{equation}
\end{theorem}
Theorem~\ref{thm:afm-sample-complexity} shows that \textsc{AFM-Elim} remains adaptive at the base-arm level even though it operates over a combinatorial family of candidate sets. An arm that appears only in clearly suboptimal candidate sets is discarded early, whereas an arm involved in difficult set comparisons or in benchmark estimation is sampled for longer. Moreover, each base-arm observation is shared across all active candidate sets containing that arm, so no separate sampling cost is incurred for individual sets. The identified representative set can further be used to recover the latent objective groups; a confidence-based assignment procedure and its sample complexity are given in Appendix~\ref{app:latent-group-recovery}.

\paragraph{Computationally Efficient Relaxation.}
Exact latent-set identification in \textsc{AFM-Elim} requires optimizing over the combinatorial family $\mathcal S_L=\{S\subseteq[K]:|S|=L\}$, which is computationally intractable in general. To obtain a polynomial-time alternative that preserves the worst-objective coverage criterion, we consider a bicriteria relaxation that allows $O(L\log M)$ representative arms. Specifically, with probability at least $1-\delta$, the relaxed procedure returns a set $\widehat S$ satisfying
\[
\Phi(\widehat S)\leq\Phi(S_L^\star)+\epsilon,
\qquad
|\widehat S|\leq\lceil H_ML\rceil,
\qquad
H_M=\sum_{j=1}^M\frac{1}{j}\leq1+\log M.
\]
The key to computational efficiency is that, at any fixed loss threshold, feasibility reduces to a set-cover problem over the objectives. Here, a loss threshold specifies the largest weighted degradation from the objective-wise ideal value that we are willing to tolerate on each objective, and an arm covers exactly those objectives whose weighted degradation does not exceed the threshold. Building on this observation, we propose \textsc{BiAFM-Greedy}, which first estimates the reward matrix by uniform sampling, then scans candidate thresholds in increasing order and applies the greedy set-cover rule at each one. It returns the first cover that satisfies the relaxed budget, thereby avoiding enumeration of all size-$L$ candidate sets. Compared with \textsc{AFM-Elim}, this procedure forgoes exact size-$L$ identification and instance-adaptive sampling, but in return offers a polynomial-time bicriteria guarantee: with high probability, its coverage loss is within $\epsilon$ of the optimal size-$L$ loss while using at most $O(L\log M)$ arms. The formal procedure and its theoretical guarantee are given in Appendix~\ref{app:computational-complexity}.

\section{Information-Theoretic Lower Bound}

The preceding algorithms come with explicit finite-time guarantees, but these guarantees do not characterize the fundamental difficulty. In this section, we establish an information-theoretic lower bound on the sample complexity of identifying the optimal representative set $S_L^\star$.

We consider Gaussian rewards, where each pull of arm $a\in[K]$ yields an independent observation $\bm r_t\sim\mathcal N(\boldsymbol\mu_a,\sigma^2\mathbf I_M)$. The Kullback--Leibler (KL) divergence between the reward distributions with mean vectors $\boldsymbol\mu_a$ and $\boldsymbol\nu_a$ is
$
d(\boldsymbol\mu_a,\boldsymbol\nu_a)
=
\frac{\|\boldsymbol\mu_a-\boldsymbol\nu_a\|_2^2}{2\sigma^2}.
$
Let $\boldsymbol\mu=(\boldsymbol\mu_1,\ldots,\boldsymbol\mu_K)$ denote a bandit instance, and let $S_L^\star(\boldsymbol\mu)$ denote its optimal representative set, which we assume to be unique. For a fixed instance $\boldsymbol\mu$, the set of alternative instances is
\begin{equation}
\operatorname{Alt}_L(\boldsymbol\mu)
=
\left\{
\boldsymbol\nu:
S_L^\star(\boldsymbol\nu)
\neq
S_L^\star(\boldsymbol\mu)
\right\}.
\end{equation}
Each $\boldsymbol\nu\in\operatorname{Alt}_L(\boldsymbol\mu)$ induces a different optimal representative set and must therefore be distinguished from $\boldsymbol\mu$ by any correct identification algorithm. The corresponding characteristic time is defined by
\begin{equation}\label{eq:latent-characteristic-time}
\left(T_L^\star(\boldsymbol\mu)\right)^{-1}
=
\sup_{\boldsymbol w\in\Delta_K}
\inf_{\boldsymbol\nu\in\operatorname{Alt}_L(\boldsymbol\mu)}
\sum_{a=1}^K
w_a
d(\boldsymbol\mu_a,\boldsymbol\nu_a).
\end{equation}
Here, the outer optimization determines the asymptotic sampling proportions, while the inner optimization selects the hardest alternative under the chosen allocation.


\begin{theorem}[Information-theoretic lower bound]
\label{thm:latent-lower-bound}
Let $\delta\in(0,1/2)$. Any fixed-confidence algorithm that identifies the optimal representative set with probability at least $1-\delta$ on every Gaussian instance with a unique optimal representative set satisfies
\begin{equation}\label{eq:latent-lower-bound}
\mathbb E_{\boldsymbol\mu}[\tau]
\geq
T_L^\star(\boldsymbol\mu)
\operatorname{KL}(1-\delta,\delta),
\end{equation}
where $\operatorname{KL}(\cdot,\cdot)$ denotes the binary relative entropy.
Consequently,
\begin{equation}\label{eq:latent-asymptotic-lower-bound}
\liminf_{\delta\downarrow0}
\frac{\mathbb E_{\boldsymbol\mu}[\tau]}
{\log(1/\delta)}
\geq
T_L^\star(\boldsymbol\mu).
\end{equation}
\end{theorem}
Theorem~\ref{thm:latent-lower-bound} shows that any $\delta$-PAC algorithm requires, in expectation, at least order $\log(1/\delta)$ samples, with $T_L^\star(\boldsymbol\mu)$ quantifying the intrinsic difficulty of the instance. The max--min form in Eq.~(\ref{eq:latent-characteristic-time}) admits the standard fixed-confidence interpretation: the learner chooses asymptotic sampling proportions $\boldsymbol w$, and the inner optimization selects the hardest alternative instance that changes the optimal representative set. The resulting weighted KL divergence gives the maximal information rate, whose inverse is the leading constant of the asymptotic sample complexity. Unlike in objective-wise MO-BAI \citep{Chen:2025MOBAI}, the lower bound only requires distinguishing alternatives that change $S_L^\star$; objective-wise near-ties are therefore irrelevant unless they alter the optimal representative set. An oracle-based Track-and-Stop algorithm that asymptotically attains this bound with high probability is presented in Appendix~\ref{app:afm-tas}.

\section{Experiments}\label{sec:experiments}
In this section, we evaluate the proposed few-for-many identification framework from both statistical and practical perspectives. We first examine whether it improves sample efficiency by avoiding unnecessary objective-wise distinctions and by sharing vector-valued observations across objectives. We then assess its practical effectiveness on a real-world task.

\paragraph{Synthetic setup.}
All synthetic experiments use Gaussian rewards with $\sigma=1$ and 100 independent trials, with random seeds shared across methods. We report the mean number of arm pulls with $95\%$ confidence intervals, and set $\delta=0.05$ unless otherwise specified. We compare against three types of baselines: uniform-sampling variants of our algorithms; \textsc{Independent-GFM}, which solves each group separately with confidence level $\delta/Q$; and the two-stage \textsc{MOBAI+GFM/AFM}, which first identifies the objective-wise best arms using the method of \citet{Chen:2025MOBAI} and then performs few-for-many identification. The two stages split the confidence budget equally, and the second stage reuses the first-stage samples. For Figure~\ref{fig:irrelevant-gap}, the known-group instance has $Q=4$ groups with four objectives each, so $M=16$; the latent-group instance has $M=9$ objectives generated from three planted blocks of size three and uses the representative budget $L=3$, but the partition is not revealed to the algorithm. These instances are constructed so that objective-wise identification becomes harder as $\xi$ decreases, while the optimal representatives remain unchanged. Each group or planted block has a common representative with mean $0.8$, each objective has a specialist with mean $1$ and a near-specialist with mean $1-\xi$, and the remaining arms are dominated with mean $0.05$. Decreasing $\xi$ therefore isolates distinctions that are difficult for MO-BAI but irrelevant to few-for-many identification. To evaluate the benefit of sharing vector observations across groups, we further use a two-arm construction with $M=Q$ singleton groups and vary $Q\in\{1,2,4,8,16,32\}$.

\begin{figure*}[t]
  \centering
  \subfigure[Known groups]{\includegraphics[width=.485\textwidth]{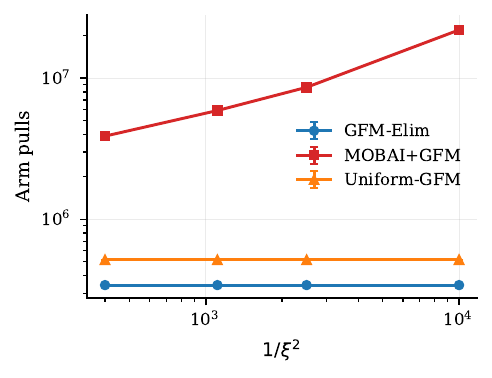}}
  \hfill
  \subfigure[Latent groups]{\includegraphics[width=.485\textwidth]{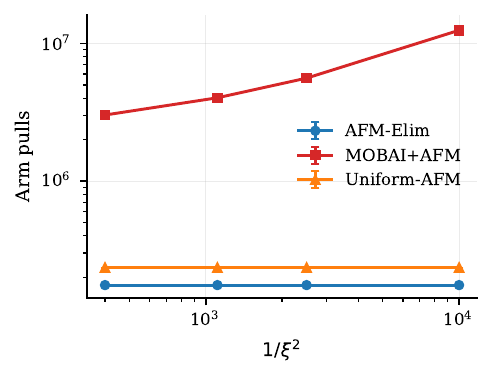}}
  \caption{Sample cost on synthetic instances with irrelevant objective-wise
  near-ties and dominated arms. Points show means over 100 trials and error bars indicate 95\% CIs; the horizontal axis is $1/\xi^2$.}
  \label{fig:irrelevant-gap}
\end{figure*}

\paragraph{Synthetic results.}
All methods identify the correct representatives in every run. In Figure~\ref{fig:irrelevant-gap}, the two-stage MO-BAI baselines require substantially more samples as $\xi$ decreases, because distinguishing each specialist from its near-specialist becomes increasingly difficult. In contrast, \textsc{GFM-Elim} and \textsc{AFM-Elim} are nearly insensitive to $\xi$, since these near-ties do not affect the optimal representatives. This illustrates the advantage of targeting the few-for-many solution directly and avoiding unnecessary objective-wise distinctions. Figure~\ref{fig:irrelevant-gap} also highlights the benefit of adaptive elimination: unlike the uniform variants, which keep sampling all arms, the proposed methods discard arms once they become irrelevant to benchmark or representative identification, and consistently require fewer samples. Figures~\ref{fig:sharing-adult}(a)--(b) further demonstrate the benefit of shared vector feedback. Because a single pull reveals all objective coordinates, \textsc{GFM-Elim} reuses the same observations across groups, whereas \textsc{Independent-GFM} solves each group separately. As a result, the relative gain grows approximately linearly in $Q$ and reaches about $32\times$ at $Q=32$, consistent with our analysis.

\begin{figure*}[t]
  \centering
  \subfigure[Sample cost]{\includegraphics[width=.31\textwidth]{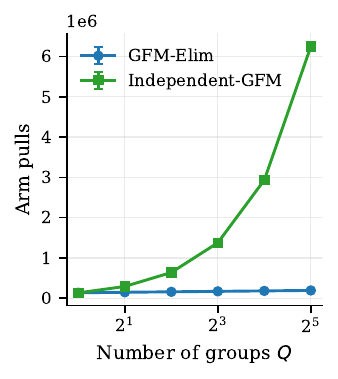}}
  \hfill
  \subfigure[Sharing gain]{\includegraphics[width=.31\textwidth]{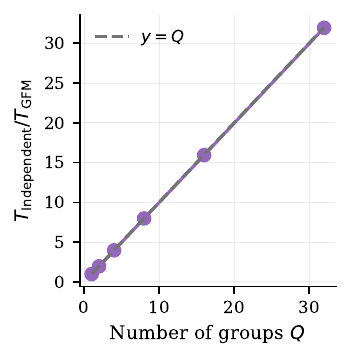}}
  \hfill
  \subfigure[Adult, $L=2$]{\includegraphics[width=.31\textwidth]{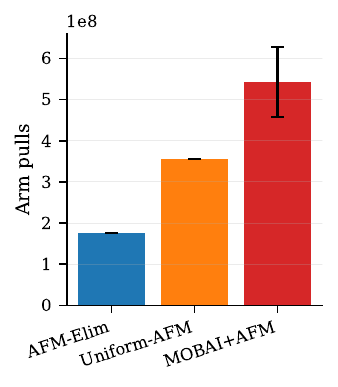}}
  \caption{(a)--(b) Sample cost and sharing gain on the shared-feedback
  synthetic instance; the dashed line in (b) marks the linear reference $y=Q$.
  (c) Sample cost on the Adult model-portfolio task; all methods identify the
  same optimal pair in every trial.}
  \label{fig:sharing-adult}
\end{figure*}

\paragraph{Real-world setup.}
We evaluate the proposed methods on the Adult dataset \citep{Kohavi:1996}. We construct 12 arms by combining three classifiers (logistic regression, random forest, and histogram gradient boosting) and four decision thresholds in $\{0.35,0.45,0.55,0.65\}$. The 15 objectives are test accuracies on demographic subgroups defined by sex, age, race, and sex--age combinations, after excluding two subgroups whose best arms are tied. Each pull returns subgroup-level correctness indicators, which are sub-Gaussian with $\sigma=1/2$. We set $L=2$, for which the optimal representative set consists of the two histogram-gradient-boosting models with thresholds $0.45$ and $0.55$. We run 100 independent trials and compare \textsc{AFM-Elim} with Uniform-AFM and \textsc{MOBAI+AFM}, using the same MO-BAI implementation as in the synthetic experiments.

\paragraph{Real-world results.}
As shown in Figure~\ref{fig:sharing-adult}(c), all methods identify the optimal two-model portfolio in every trial, but their sampling costs differ substantially. Uniform-AFM samples every model throughout the elimination process and thus cannot exploit the fact that several classifiers are quickly revealed to be irrelevant to the best size-two portfolio. \textsc{AFM-Elim} eliminates such arms and candidate sets earlier, reducing the cost by roughly a factor of two. The MO-BAI-based baseline is even more expensive, because its first stage attempts to identify the best model for every subgroup before solving the portfolio problem. On Adult, several subgroup accuracies are close across neighboring thresholds, so resolving every subgroup-wise winner requires many additional samples. The few-for-many objective, in contrast, only needs to determine which pair jointly covers the subgroups well, and can therefore stop without resolving all objective-wise winners. Consequently, \textsc{AFM-Elim} uses only about one third of the pulls of \textsc{MOBAI+AFM}. Overall, the real-world experiment supports the same conclusion as the synthetic study: directly targeting the deployable representative set can be much cheaper than solving the stronger objective-wise identification problem.

Appendix~\ref{app:additional-experiments} provides further empirical results, including confidence calibration, objective scaling, representative-gap sensitivity, and a computational study of the bicriteria approximation algorithm.

\section{Conclusion and Future Work}
We formulated few-for-many best-arm identification as a fixed-confidence problem whose goal is to find a small set of arms such that every objective is well served by at least one selected arm. We considered two settings: known objective groups, where one representative is identified for each group, and latent objective groups, where a budgeted representative set is selected without a predefined partition. For these settings, we developed \textsc{GFM-Elim} and \textsc{AFM-Elim}, which exploit vector-valued feedback and eliminate arms as soon as they become irrelevant to the final decision. We established instance-dependent sample-complexity guarantees for both algorithms and an information-theoretic lower bound for the latent-group setting, and further showed that this lower bound is asymptotically attainable with high probability by an oracle-based Track-and-Stop procedure. Our experiments demonstrate that directly targeting the few-for-many solution can substantially reduce the sample complexity by sharing observations and avoiding unnecessary objective-wise distinctions.

An important direction for future work is to better reconcile statistical efficiency with computational scalability in the latent-group setting, where exact identification involves combinatorial optimization. It would also be of interest to extend the framework beyond the current setting, for example, to adaptive selection of the representative budget and to partial feedback.



\newpage
\appendix

\begin{center}
\LARGE \textbf{Appendix}
\end{center}

\vspace{0.5cm}

\section{Extended Related Work}\label{app:extended-related-work}
In this section, we review three lines of research most relevant to our work: fixed-confidence best-arm identification, multi-objective bandits, and few-for-many optimization.

\paragraph{Fixed-confidence best-arm identification.}
Fixed-confidence best-arm identification (BAI) seeks to identify the optimal arm with a prescribed confidence level while minimizing the number of samples \citep{Audibert:2010,Gabillon:2012}. A major line of work studies its instance-dependent sample complexity and develops sampling strategies that attain the corresponding information-theoretic limits. In the classical setting, \citet{Garivier:2016} established an asymptotically tight lower bound and proposed Track-and-Stop, which adaptively tracks the optimal sampling proportions. This framework has since been extended to more structured pure-exploration problems, including combinatorial identification \citep{Chen:2014CPE,Chen:2016Matroid,Jourdan:2021}, identification with multiple correct answers \citep{Degenne:2019}, and linear and transductive bandits \citep{Fiez:2019,Degenne:2020,Jedra:2020}. These studies show that the structure of the answer space and of the feedback model plays a central role in determining both the optimal sampling allocation and the resulting sample complexity.

\paragraph{Multi-objective bandits.}
Multi-objective bandits extend classical bandit models to vector-valued rewards and have been studied extensively in both online decision making and pure exploration \citep{Drugan:2013}. In the online setting, one line of work targets Pareto-efficient decisions without imposing an explicit scalar preference, including contextual Pareto regret \citep{Turgay:2018}, Pareto regret under stochastic and adversarial feedback \citep{Xu:2023}, and Pareto-regret minimization in linear contextual bandits \citep{Park:2025}. Another line incorporates user preferences through scalarizations or structured orderings. Representative examples include nonlinear scalarizations for exploring the Pareto front with hypervolume guarantees \citep{Zhang:2024}, as well as lexicographic and Pareto--lexicographic bandits that optimize objectives according to prescribed priority structures \citep{Xue:2024,Cheng:2024}. Recent extensions study cost-aware multi-objective evaluation \citep{xue2026costaware}, lexicographic generalized low-rank matrix bandits \citep{xue2026efficientmatrix}, safe multi-objective linear bandits under hierarchical preferences \citep{xue2026safe}, and multi-objective reinforcement learning with bandit feedback \citep{qiu2024traversing}.

In pure exploration, most existing work focuses on identifying the Pareto-optimal set under fixed-budget or fixed-confidence criteria \citep{Auer:2016,Ararat:2023,Cyrille:2023,Crepon:2024}. More recent studies incorporate preference information through ordering cones, which induce partial orders over vector-valued rewards and encode structured trade-offs among objectives \citep{Shukla:2024PrePEx,Das:2025FraPPE}. Most closely related to our work, \citet{Chen:2025MOBAI} study fixed-confidence multi-objective best-arm identification, in which the learner identifies the best arm for each objective, and characterize asymptotically optimal sampling. In contrast, few-for-many identification seeks only a small set of representative arms that collectively serve all objectives. Consequently, it need not recover the entire Pareto set or resolve every objective-wise best arm, but only those distinctions that can affect the optimal representatives.

\paragraph{Few-for-many optimization.}
The few-for-many paradigm has emerged as an alternative to approximating a large Pareto set when only a small number of solutions can ultimately be deployed. \citet{Liu:2024MaCP} formulate the many-objective cover problem (MaCP), which selects $k\ll M$ solutions such that each objective is well served by at least one of them, and propose a clustering-based method for the resulting NP-hard problem. Building on this idea, \citet{Lin:2025few} introduce Tchebycheff set scalarization, which evaluates a set through its worst objective-wise performance, together with a smooth variant amenable to gradient-based optimization. \citet{Liu:2025F4M} further develop a conditional set-optimization framework that supports different scalarizations and both continuous and discrete many-objective problems. Subsequent studies improve the computational efficiency of few-for-many optimization \citep{Shang:2026F4M} and extend the paradigm to personalized federated learning, where a small collection of shared models is optimized to serve heterogeneous clients \citep{Guo:2026FedFew}.

These works primarily address deterministic optimization, where objective evaluations or optimization feedback are directly available. In contrast, we study few-for-many identification under statistical uncertainty, where the vector-valued means of the candidate arms are unknown and must be inferred from noisy sequential observations.

\section{Polynomial-Time Bicriteria Approximate Identification}
\label{app:computational-complexity}

\begin{algorithm}[tb]
\caption{\textsc{GreedyCover}}
\label{alg:greedy-cover}
\begin{algorithmic}[1]
\Require Coverage sets $\{C_a\}_{a\in[K]}$ and universe $[M]$
\State $S\gets\emptyset$, $\mathcal U\gets[M]$, and $\mathcal A\gets[K]$
\While{$\mathcal U\neq\emptyset$}
    \State Choose $a^\prime\in\arg\max_{a\in\mathcal A}|C_a\cap\mathcal U|$
    \If{$|C_{a^\prime}\cap\mathcal U|=0$}
        \State \Return \textsc{Fail}
    \EndIf
    \State $S\gets S\cup\{a^\prime\}$, $\mathcal U\gets\mathcal U\setminus C_{a^\prime}$, and $\mathcal A\gets\mathcal A\setminus\{a^\prime\}$
\EndWhile
\State \Return $S$
\end{algorithmic}
\end{algorithm}

\begin{algorithm}[tb]
\caption{\textsc{BiAFM-Greedy}: Bicriteria Approximate Few-for-Many Identification}
\label{alg:biafm-greedy}
\begin{algorithmic}[1]
\Require Accuracy $\epsilon>0$, confidence $\delta\in(0,1)$, budget $L$, and weights $\{\lambda_m\}_{m\in[M]}$
\State $\overline\lambda\gets\max_{m\in[M]}\lambda_m$ and $n\gets\left\lceil\frac{32\sigma^2\overline\lambda^2}{\epsilon^2}\log\left(\frac{2KM}{\delta}\right)\right\rceil$
\State Pull every arm $n$ times and compute $\widehat\mu_{a,m}$ for all $a,m$
\State $\widehat\mu_m^\star\gets\max_{a\in[K]}\widehat\mu_{a,m}$ for every $m\in[M]$
\State $\widehat d_{a,m}\gets\lambda_m(\widehat\mu_m^\star-\widehat\mu_{a,m})$ for all $a,m$
\State $\widehat{\mathcal T}\gets\{\widehat d_{a,m}:a\in[K],m\in[M]\}$, sorted increasingly after removing duplicates
\State $B_L\gets\min\{K,\lceil H_ML\rceil\}$
\For{each $\gamma\in\widehat{\mathcal T}$ in increasing order}
    \State $\widehat C_a(\gamma)\gets\{m\in[M]:\widehat d_{a,m}\leq\gamma\}$ for every $a\in[K]$
    \State $S_\gamma\gets\textsc{GreedyCover}(\{\widehat C_a(\gamma)\}_{a\in[K]},[M])$
    \If{$S_\gamma\neq\textsc{Fail}$ and $|S_\gamma|\leq B_L$}
        \State \Return $\widehat S\gets S_\gamma$
    \EndIf
\EndFor
\end{algorithmic}
\end{algorithm}

The exact \textsc{AFM-Elim} algorithm identifies $S_L^\star$ with an instance-dependent sample complexity, but it maintains a combinatorial family of candidate sets. This section presents \textsc{BiAFM-Greedy}, a complementary polynomial-time procedure. The price of computational tractability is a bicriteria relaxation: the returned set may contain $O(L\log M)$ arms, while its coverage loss remains within $\epsilon$ of the optimal size-$L$ loss.

The key observation is that, for a fixed target loss, the latent few-for-many problem reduces to a set-cover problem over the objectives. For each arm $a\in[K]$ and objective $m\in[M]$, define the weighted deficit $d_{a,m}=\lambda_m(\mu_m^\star-\mu_{a,m})$. The latent-set loss can then be written as
\begin{equation*}
\Phi(S)
=
\max_{m\in[M]}
\min_{a\in S}
d_{a,m}.
\end{equation*}
For any threshold $\gamma\geq0$, define the set of objectives covered by arm $a$ as
\begin{equation*}
C_a(\gamma)
=
\{m\in[M]:d_{a,m}\leq\gamma\}.
\end{equation*}
It follows that
\begin{equation*}
\Phi(S)\leq\gamma
\quad\Longleftrightarrow\quad
\bigcup_{a\in S}C_a(\gamma)=[M].
\end{equation*}
Hence, for a fixed threshold $\gamma$, finding a minimum-cardinality representative set with loss at most $\gamma$ is equivalent to finding a minimum set cover of the objectives.

We therefore use the standard greedy set-cover procedure in Algorithm~\ref{alg:greedy-cover}. Starting with all objectives uncovered, it repeatedly selects the arm that covers the largest number of currently uncovered objectives. Whenever a feasible cover exists, the classical greedy guarantee gives
\begin{equation*}
|S_{\rm greedy}|
\leq
H_M\,\operatorname{OPT}_{\rm SC},
\qquad
H_M=\sum_{j=1}^M\frac{1}{j}\leq1+\log M,
\end{equation*}
where $\operatorname{OPT}_{\rm SC}$ denotes the minimum cover size.

The complete \textsc{BiAFM-Greedy} procedure is summarized in Algorithm~\ref{alg:biafm-greedy}. It first samples all arms uniformly to estimate the unknown means, and then computes the corresponding empirical ideal values and weighted deficits. Instead of enumerating the candidate family $\mathcal S_L$, the algorithm uses the distinct empirical deficits as candidate thresholds. For each threshold, it constructs the induced empirical coverage sets and applies \textsc{GreedyCover}; the first threshold whose greedy cover satisfies the relaxed budget $B_L=\min\{K,\lceil H_ML\rceil\}$ determines the returned set. Algorithm~\ref{alg:biafm-greedy} thus examines at most $KM$ thresholds and never enumerates the $\binom{K}{L}$ candidate sets.

\begin{theorem}[Polynomial-time bicriteria identification]\label{thm:biafm-greedy}
Suppose that the reward noise is $\sigma^2$-sub-Gaussian, and let $\overline\lambda=\max_{m\in[M]}\lambda_m$. For any $\epsilon>0$ and $\delta\in(0,1)$, with probability at least $1-\delta$, Algorithm~\ref{alg:biafm-greedy} returns a set $\widehat S$ such that
\begin{equation*}
|\widehat S|\leq\min\{K,\lceil H_ML\rceil\},
\qquad
\Phi(\widehat S)\leq\Phi(S_L^\star)+\epsilon,
\end{equation*}
where $H_M=\sum_{j=1}^M1/j\leq1+\log M$. The total number of arm pulls satisfies
\begin{equation}\label{eq:biafm-sample-complexity}
T_{\rm BiAFM}
\leq
K\left\lceil\frac{32\sigma^2\overline\lambda^2}{\epsilon^2}\log\left(\frac{2KM}{\delta}\right)\right\rceil.
\end{equation}
Moreover, the computational complexity of the algorithm is polynomial in $K$, $M$, $L$, $1/\epsilon$, and $\log(1/\delta)$.
\end{theorem}

Theorem~\ref{thm:biafm-greedy} provides a finite-time bicriteria guarantee for the polynomial-time relaxation. Ignoring integer rounding, it is an $(H_M,\epsilon,\delta)$-PAC result: with probability at least $1-\delta$, the algorithm returns a set whose loss is within $\epsilon$ of that of the best size-$L$ representative set, at the cost of a multiplicative relaxation of at most $H_M\leq1+\log M$ in set size. Thus, \textsc{BiAFM-Greedy} complements \textsc{AFM-Elim} by trading exact cardinality for polynomial-time computation.

\section{Oracle-Based Track-and-Stop}\label{app:afm-tas}

This section presents \textsc{AFM-TaS}, an oracle-based, asymptotically optimal counterpart of \textsc{AFM-Elim}. Rather than providing a practical finite-time procedure, its purpose is to show that, under Gaussian rewards and given the required optimization oracles, the information-theoretic lower bound in Theorem~\ref{thm:latent-lower-bound} is asymptotically attainable with high probability. Following the Track-and-Stop framework \citep{Garivier:2016,Chen:2025MOBAI}, \textsc{AFM-TaS} repeatedly estimates the optimal representative set, tracks the corresponding optimal sampling allocation, and stops once sufficient evidence separates the current answer from all alternatives.

The optimal sampling proportions are characterized by the max--min problem in Eq.~(\ref{eq:latent-characteristic-time}). For an instance $\boldsymbol\mu$, define
\begin{equation*}
\mathcal W_L^\star(\boldsymbol\mu)
=
\arg\max_{\boldsymbol w\in\Delta_K}
\inf_{\boldsymbol\nu\in\operatorname{Alt}_L(\boldsymbol\mu)}
\sum_{a=1}^K
w_a
d(\boldsymbol\mu_a,\boldsymbol\nu_a).
\end{equation*}
At round $t$, let $\widehat{\boldsymbol\mu}(t)$ denote the empirical mean instance. The current empirical answer is
\begin{equation}\label{eq:empirical-representative-set}
\widehat S_t
\in
\arg\min_{S\in\mathcal S_L}
\Phi_{\widehat{\boldsymbol\mu}(t)}(S),
\end{equation}
and the statistical evidence in its favor is quantified by the generalized likelihood-ratio (GLR) statistic
\begin{equation}\label{eq:latent-glr-statistic}
Z_L(t)
=
\inf_{\boldsymbol\nu:\,
S_L^\star(\boldsymbol\nu)\neq\widehat S_t}
\sum_{a=1}^K
N_a(t)
d\left(
\widehat{\boldsymbol\mu}_a(t),
\boldsymbol\nu_a
\right).
\end{equation}
Thus, $Z_L(t)$ measures the minimum accumulated evidence needed to move from the empirical instance to an alternative under which $\widehat S_t$ is no longer optimal.

We instantiate the stopping threshold as
\begin{equation*}
\beta(t,\delta)
=
\log\left(\frac{\pi^2t^2|\mathcal S_L|}{6\delta}\right),
\end{equation*}
where the factor $|\mathcal S_L|$ accounts for the finite answer space. By the standard GLR maximal inequality and a union bound over the candidate answers and over time (using $\sum_{t\geq1}6/(\pi^2t^2)=1$), this threshold satisfies
\begin{equation}\label{eq:valid-glr-threshold}
\mathbb P_{\boldsymbol\mu}
\left(
\exists t\geq1:
\widehat S_t\neq S_L^\star(\boldsymbol\mu),
\;
Z_L(t)\geq\beta(t,\delta)
\right)
\leq\delta,
\end{equation}
and, for fixed problem dimensions,
\begin{equation}\label{eq:threshold-asymptotics}
\beta(t,\delta)
=
\log\frac{1}{\delta}
+
o\left(\log\frac{1}{\delta}\right)
\end{equation}
uniformly over $t=O(\log(1/\delta))$.

\textsc{AFM-TaS} first samples every arm once. At each subsequent round, it stops and returns $\widehat S_t$ if $Z_L(t)\geq\beta(t,\delta)$. Otherwise, it computes an allocation in $\mathcal W_L^\star(\widehat{\boldsymbol\mu}(t))$ and tracks it while enforcing sufficient exploration of every arm. This tracking rule drives the empirical sampling proportions toward the information-theoretically optimal allocation while ensuring consistent estimation of the underlying instance.

\begin{algorithm}[tb]
\caption{\textsc{AFM-TaS}: Oracle-Based Adaptive Few-for-Many Track-and-Stop}
\label{alg:afm-tas}
\begin{algorithmic}[1]
\Require Confidence level $\delta$ and representative budget $L$
\State Pull every arm once
\For{$t=K,K+1,\ldots$}
    \State Compute $\widehat S_t$ and $Z_L(t)$ using
    Eqs.~(\ref{eq:empirical-representative-set})--(\ref{eq:latent-glr-statistic})
    \If{$Z_L(t)\geq\beta(t,\delta)$}
        \State \Return $\widehat S_t$
    \EndIf
    \State Compute
    $
    \boldsymbol w_t
    \in
    \mathcal W_L^\star
    \left(
    \widehat{\boldsymbol\mu}(t)
    \right)
    $
    \If{$\min_{a\in[K]}N_a(t)<\sqrt{t}-K/2$}
        \State Pull an arm with the smallest number of observations
    \Else
        \State Pull
        $
        A_{t+1}
        \in
        \arg\min_{a\in[K]}
        \left[
        N_a(t)
        -
        \sum_{s=K}^{t}w_{s,a}
        \right]
        $
    \EndIf
\EndFor
\end{algorithmic}
\end{algorithm}

\begin{theorem}[Oracle-based asymptotic upper bound]\label{thm:afm-asymptotic-optimality}
Suppose that the rewards are Gaussian and $T_L^\star(\boldsymbol\mu)<\infty$, and that the allocation oracle in Step 7 is exact. Then Algorithm~\ref{alg:afm-tas} is $\delta$-PAC and, for every $\eta>0$, satisfies
\begin{equation}\label{eq:afm-asymptotic-upper}
\lim_{\delta\downarrow0}
\mathbb P_{\boldsymbol\mu}
\left(
\tau_\delta
\leq
(1+\eta)T_L^\star(\boldsymbol\mu)
\log\frac{1}{\delta}
\right)
=1.
\end{equation}
Thus, Algorithm~\ref{alg:afm-tas} attains the characteristic time
$T_L^\star(\boldsymbol\mu)$ with asymptotically high probability.
\end{theorem}

Theorem~\ref{thm:afm-asymptotic-optimality} shows that the oracle allocation is not merely a heuristic for focusing samples: with asymptotically high probability, the stopping time of \textsc{AFM-TaS} matches the characteristic time $T_L^\star(\boldsymbol\mu)$ appearing in the lower bound of Theorem~\ref{thm:latent-lower-bound}. By contrast, \textsc{AFM-Elim} trades this asymptotic optimality for a simpler elimination rule and an explicit non-asymptotic guarantee. The two algorithms therefore play different roles: \textsc{AFM-Elim} is the implementable finite-time procedure, whereas \textsc{AFM-TaS} certifies the information-theoretic benchmark.

\section{Additional Experiments}\label{app:additional-experiments}

This appendix provides additional experiments that complement the main results in Section~\ref{sec:experiments}. We first examine confidence calibration, scaling with the number of objectives, and sensitivity to the gaps that directly determine the few-for-many answer. We then study the computational behavior of the polynomial-time bicriteria algorithm in Appendix~\ref{app:computational-complexity}. We do not include \textsc{AFM-TaS}, since it is an oracle-based construction that requires solving the characteristic-time allocation problem exactly and repeatedly computing GLR stopping statistics against all alternative representative sets.

\subsection{Confidence Calibration}

We first examine whether the empirical error rate is consistent with the prescribed confidence level. We consider one known-group instance and one latent-group instance from the synthetic setting in Section~\ref{sec:experiments}, fix $\xi=0.1$, and vary $\delta\in\{0.2,0.1,0.05,0.02\}$. For each $\delta$, we run 200 independent trials, using \textsc{GFM-Elim} on the known-group instance and \textsc{AFM-Elim} on the latent-group instance.

\begin{figure}[H]
  \centering
  \includegraphics[width=.86\textwidth]{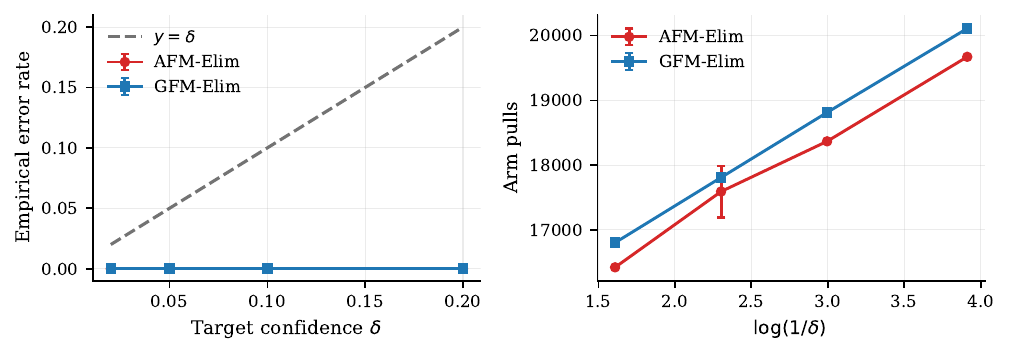}
  \caption{Confidence calibration. The left panel compares the empirical error rate with the target confidence level, and the right panel reports the
  corresponding sample complexity.}
  \label{fig:appendix-confidence-calibration}
\end{figure}

As shown in Figure~\ref{fig:appendix-confidence-calibration}, both algorithms return the correct answer in all trials at every tested confidence level, which is consistent with the conservative finite-time confidence allocation used by the elimination procedures. The sample complexity increases only mildly as $\delta$ decreases, in line with the logarithmic dependence on $1/\delta$ in the theoretical bounds. Hence, the algorithms remain reliable across confidence levels, and the sample budget grows gradually as the confidence requirement becomes stricter.

\subsection{Scaling with the Number of Objectives}

We next isolate the many-objective effect in the known-group setting. We fix $Q=4$ groups and increase the number of objectives per group from $2$ to $16$, so that $M$ ranges from $8$ to $64$. The optimal representative of each group remains unchanged, while every objective has its own specialist and near-specialist. Consequently, methods that first identify all objective-wise best arms must resolve a growing number of distinctions that are unnecessary for the few-for-many answer.

\begin{figure}[H]
  \centering
  \includegraphics[width=.52\textwidth]{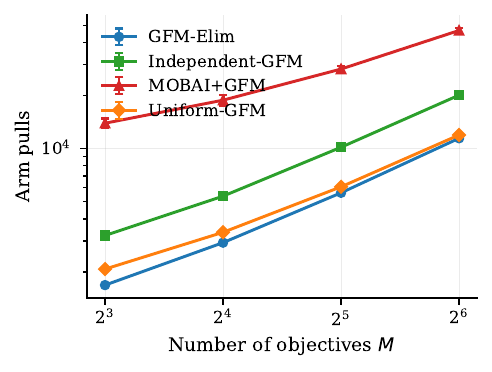}
  \caption{Scaling with the number of objectives in the known-group setting.}
  \label{fig:appendix-objective-scaling}
\end{figure}

As shown in Figure~\ref{fig:appendix-objective-scaling}, \textsc{GFM-Elim} consistently identifies the group representatives using substantially fewer pulls than the two-stage \textsc{MOBAI+GFM} baseline. The gap widens as $M$ grows, because \textsc{MOBAI+GFM} spends many samples resolving objective-wise specialists before solving the grouped problem. \textsc{GFM-Elim} also improves over \textsc{Independent-GFM} by sharing each vector observation across group decisions. Its curve approaches that of Uniform-GFM for larger $M$, because this experiment adds a specialist and a near-specialist arm for each new objective while keeping the number of easily discarded dominated arms fixed; hence, a growing fraction of arms remains relevant to benchmark estimation. These results confirm the main message of the synthetic study: as the number of objectives grows, directly identifying representatives is more sample-efficient than recovering every objective-wise winner.

\subsection{Sensitivity to Representative Gaps}

The main experiment in Figure~\ref{fig:irrelevant-gap} varies an objective-wise gap that is irrelevant to the answer. Here, we instead vary the gap between the optimal representative and its closest competitor, while keeping the near-specialist gap fixed at $\xi=0.1$. This directly changes the difficulty of the few-for-many decision.

\begin{figure}[t]
  \centering
  \includegraphics[width=.86\textwidth]{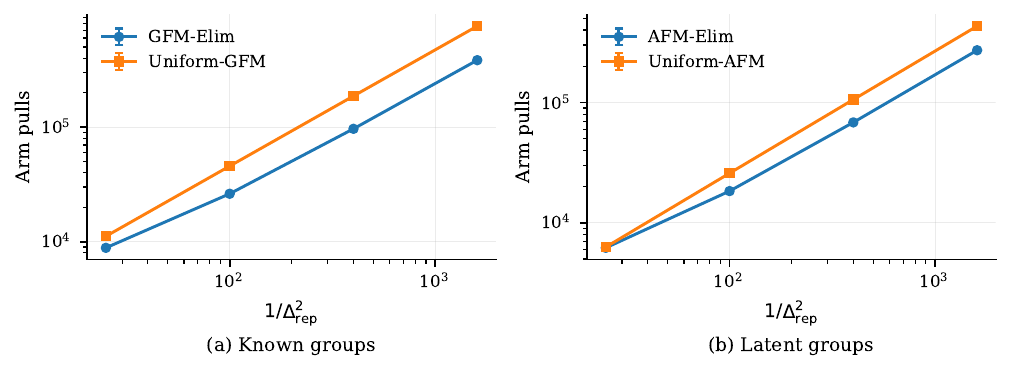}
  \caption{Sensitivity to the representative gap. Both known-group and latent-group identification become harder as the representative gap shrinks.}
  \label{fig:appendix-representative-gap}
\end{figure}

Figure~\ref{fig:appendix-representative-gap} exhibits the expected quadratic growth in sample complexity as the representative gap decreases. In the known-group case, the average sample complexity of \textsc{GFM-Elim} increases from the order of thousands to the order of hundreds of thousands as the representative gap shrinks from $0.2$ to $0.025$; the latent-group results for \textsc{AFM-Elim} follow the same pattern. This contrasts with Figure~\ref{fig:irrelevant-gap}: the proposed methods are insensitive to gaps that do not change the few-for-many answer, yet correctly spend more samples on gaps that determine the representative decision. Together, the two experiments indicate that the algorithms adapt to the right notion of difficulty.

\subsection{Computational Scalability of the Bicriteria Algorithm}
We further evaluate the polynomial-time \textsc{BiAFM-Greedy} algorithm. This experiment focuses on computational efficiency and complements, rather than replaces, the sample-complexity comparisons in Section~\ref{sec:experiments}. Unlike \textsc{AFM-Elim}, which maintains the combinatorial candidate family to identify the exact size-$L$ representative set, \textsc{BiAFM-Greedy} uses uniform sampling and returns a representative set under the relaxed budget.

We use the latent version of the synthetic family from Section~\ref{sec:experiments}, with $Q=4$ planted blocks of three objectives each ($M=12$), $L=4$, $\xi=0.05$, $\sigma=0.1$, and $\delta=0.05$. We vary the number of dominated arms, which increases $K$ and hence the number of exact candidate sets $|\mathcal S_L|=\binom{K}{L}$. For \textsc{BiAFM-Greedy}, we use $\epsilon=0.1$, which gives the relaxed budget $\lceil H_{12}L\rceil=13$. Each configuration is averaged over 30 independent trials with common random seeds. Reported runtimes exclude the post-hoc computation of the true optimum, which is used only for evaluation.

\begin{table}[t]
\caption{Supplementary latent-group experiment for the polynomial-time bicriteria procedure. \textsc{AFM-Elim} returns the exact size-$L$ optimum, whereas \textsc{BiAFM-Greedy} avoids enumerating $\mathcal S_L$ and returns a bicriteria set within the relaxed budget.}
\centering
\begin{tabular}{llrrrrr}
\toprule
$K$ & Method & $|\mathcal S_L|$ & Pulls & Time (ms) & $|\widehat S|$ & $\Phi(\widehat S)$ \\
\midrule
32 & \textsc{AFM-Elim} & 35,960 & 2,144 & 42.9 & 4 & 0.20 \\
32 & \textsc{BiAFM-Greedy} & 35,960 & 9,888 & 0.35 & 12 & 0.00 \\
48 & \textsc{AFM-Elim} & 194,580 & 2,480 & 238.5 & 4 & 0.20 \\
48 & \textsc{BiAFM-Greedy} & 194,580 & 15,456 & 0.43 & 12 & 0.00 \\
64 & \textsc{AFM-Elim} & 635,376 & 2,784 & 782.8 & 4 & 0.20 \\
64 & \textsc{BiAFM-Greedy} & 635,376 & 21,184 & 0.51 & 12 & 0.00 \\
\bottomrule
\end{tabular}
\label{tab:appendix-biafm}
\end{table}

Table~\ref{tab:appendix-biafm} summarizes the statistical--computational trade-off. On this instance, exact \textsc{AFM-Elim} is more sample-efficient: it respects the strict budget $L=4$ and uses only the samples needed to certify the exact optimal representative set. However, its runtime grows quickly with $K$, because the candidate family expands from tens of thousands to hundreds of thousands of sets. In contrast, \textsc{BiAFM-Greedy} samples all arms uniformly and therefore uses more pulls, but it never enumerates $\mathcal S_L$, and its runtime stays below one millisecond in all tested configurations. The returned set sizes and coverage losses reflect the bicriteria nature of the guarantee: \textsc{BiAFM-Greedy} returns a larger set, still within the relaxed logarithmic budget, which attains zero true coverage loss in this construction, whereas \textsc{AFM-Elim} returns the best strict size-$L$ set, whose coverage loss is nonzero.

\begin{figure}[t]
  \centering
  \includegraphics[width=.86\textwidth]{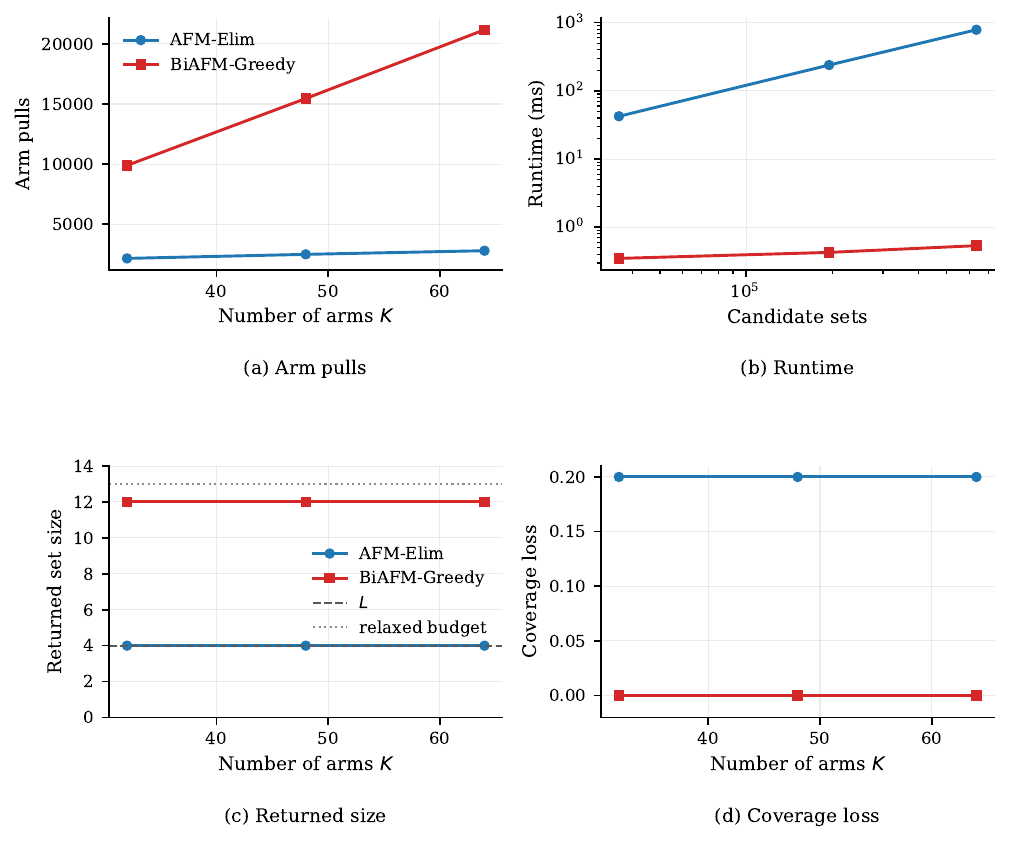}
  \caption{\textsc{BiAFM-Greedy} versus \textsc{AFM-Elim}. The four panels show arm pulls, runtime, returned set size, and true coverage loss as the exact candidate family grows.}
  \label{fig:appendix-biafm}
\end{figure}

Figure~\ref{fig:appendix-biafm} visualizes these trends across the four measured quantities. The sample panel shows the price of the uniform exploration in \textsc{BiAFM-Greedy}, while the runtime panel shows the benefit of replacing exact enumeration of candidate sets with a greedy cover computation. The set-size and coverage-loss panels make the approximation trade-off explicit: the bicriteria method returns more arms, and this enlarged set covers all objectives better than the strict size-$L$ optimum on this instance. Hence, \textsc{BiAFM-Greedy} is useful when exact enumeration of candidate sets is the computational bottleneck and a modest relaxation of the budget is acceptable.

\section{Analysis of the Grouped Method: \textsc{GFM-Elim}}
\subsection{Confidence bounds and correctness}

For every arm $a$, objective $m$, and phase $r$, let $\widehat\mu_{a,m}(n_r)$ denote the empirical mean of the first $n_r$ samples of arm $a$ on objective $m$. Consider the event
\begin{equation}\label{eq:good-event}
\mathcal E
=
\bigcap_{r=1}^{\infty}
\bigcap_{a\in[K]}
\bigcap_{m\in[M]}
\left\{
\left|
\widehat\mu_{a,m}(n_r)-\mu_{a,m}
\right|
\leq
\epsilon_r
\right\}.
\end{equation}

\begin{lem}\label{lem:confidence-event}
Under Eq.~(\ref{eq:subgaussian-assumption}),
$
\mathbb P_{\boldsymbol\mu}(\mathcal E)\geq1-\delta.
$
\end{lem}

\begin{proof}
For every fixed $r$, $a$, and $m$, the sub-Gaussian concentration inequality gives
\[
\mathbb P_{\boldsymbol\mu}
\left(
\left|
\widehat\mu_{a,m}(n_r)-\mu_{a,m}
\right|
>
\epsilon_r
\right)
\leq
2\exp\left(
-\frac{n_r\epsilon_r^2}{2\sigma^2}
\right)
\leq
\frac{\delta_r}{KM}.
\]
A union bound over all arms and objectives yields
\[
\mathbb P_{\boldsymbol\mu}
\left(
\exists a\in[K],m\in[M]:
\left|
\widehat\mu_{a,m}(n_r)-\mu_{a,m}
\right|
>
\epsilon_r
\right)
\leq
\delta_r.
\]
Finally,
\[
\sum_{r=1}^{\infty}\delta_r
=
\frac{6\delta}{\pi^2}
\sum_{r=1}^{\infty}\frac{1}{r^2}
=
\delta,
\]
and a union bound over phases completes the proof.
\end{proof}

\begin{lem}\label{lem:loss-confidence}
On the event $\mathcal E$, the following statements hold at every phase $r$.

\begin{enumerate}
    \item The best arm of every objective remains in the corresponding benchmark set:
    \[
    b_m^\star\in\mathcal B_{m,r},
    \qquad
    \forall m\in[M].
    \]
    \item For every unresolved group $q$ and every $a\in\mathcal C_{q,r}$,
    \begin{equation}\label{eq:loss-ci-valid}
    \underline R_{q,r}(a)
    \leq
    R_q(a)
    \leq
    \overline R_{q,r}(a).
    \end{equation}
    \item The loss bounds satisfy
    \begin{equation}\label{eq:loss-ci-width}
    R_q(a)-4\overline\lambda_q\epsilon_r
    \leq
    \underline R_{q,r}(a)
    \leq
    \overline R_{q,r}(a)
    \leq
    R_q(a)+4\overline\lambda_q\epsilon_r.
    \end{equation}
\end{enumerate}
\end{lem}

\begin{proof}
We first prove the benchmark-set property by induction. Initially, $\mathcal B_{m,1}=[K]$, so the claim holds. Suppose that $b_m^\star\in\mathcal B_{m,r}$. On $\mathcal E$,
\[
\widehat\mu_{b_m^\star,m}^{(r)}
\geq\mu_m^\star-\epsilon_r,
\qquad
\widehat\mu_{m,r}^\star
\leq\mu_m^\star+\epsilon_r.
\]
Hence, $\widehat\mu_{b_m^\star,m}^{(r)}\geq\widehat\mu_{m,r}^\star-2\epsilon_r$, and thus $b_m^\star\in\mathcal B_{m,r+1}$.

Since $b_m^\star\in\mathcal B_{m,r}$, the definition of the maximum and the event $\mathcal E$ imply
\begin{equation}\label{eq:empirical-ideal-concentration}
\left|\widehat\mu_{m,r}^\star-\mu_m^\star\right|\leq\epsilon_r.
\end{equation}
Consequently, for every arm $a$,
\[
\left|
\left(\widehat\mu_{m,r}^\star-\widehat\mu_{a,m}^{(r)}\right)
-
\left(\mu_m^\star-\mu_{a,m}\right)
\right|
\leq2\epsilon_r.
\]
Since the positive-part operator is nondecreasing and $1$-Lipschitz, and $\mu_m^\star-\mu_{a,m}\geq0$, subtracting and adding $2\epsilon_r$ as in Eq.~(\ref{eq:group-loss-estimated}) and taking the weighted maximum over $m\in G_q$ proves Eq.~(\ref{eq:loss-ci-valid}). The same display also yields
\[
\underline R_{q,r}(a)
\geq R_q(a)-4\overline\lambda_q\epsilon_r,
\qquad
\overline R_{q,r}(a)
\leq R_q(a)+4\overline\lambda_q\epsilon_r,
\]
which proves Eq.~(\ref{eq:loss-ci-width}).
\end{proof}

\begin{lem}[Elimination correctness]\label{thm:gse-correctness}
Algorithm~\ref{alg:gse} satisfies
\[
\mathbb P_{\boldsymbol\mu}
\left(
\tau<\infty,
\widehat{\bm a}_{\mathcal G}
=
\bm a_{\mathcal G}^\star
\right)
\geq
1-\delta.
\]
More precisely, on the event $\mathcal E$:

\begin{enumerate}
    \item $a_q^\star$ is never eliminated from $\mathcal C_{q,r}$;
    \item every $a\neq a_q^\star$ is eliminated from $\mathcal C_{q,r}$ at the first phase $r$ at which
    \[
    \epsilon_r
    <
    \frac{\Delta_{q,a}}{8\overline\lambda_q}
    \]
    and $a$ is still active in group $q$.
\end{enumerate}
\end{lem}

\begin{proof}
Fix a group $q$. By Lemma~\ref{lem:loss-confidence},
\[
\underline R_{q,r}(a_q^\star)
\leq
R_q(a_q^\star).
\]
Moreover, for every $b\in\mathcal C_{q,r}$,
\[
\overline R_{q,r}(b)
\geq
R_q(b)
\geq
R_q(a_q^\star).
\]
Therefore,
\[
\underline R_{q,r}(a_q^\star)
\leq
\min_{b\in\mathcal C_{q,r}}
\overline R_{q,r}(b),
\]
so $a_q^\star$ is never eliminated.

Now consider an arm $a\neq a_q^\star$. By Eq.~(\ref{eq:loss-ci-width}),
\[
\underline R_{q,r}(a)
\geq
R_q(a)-4\overline\lambda_q\epsilon_r
\]
and, since $a_q^\star\in\mathcal C_{q,r}$,
\[
\min_{b\in\mathcal C_{q,r}}
\overline R_{q,r}(b)
\leq
\overline R_{q,r}(a_q^\star)
\leq
R_q(a_q^\star)+4\overline\lambda_q\epsilon_r.
\]
Hence,
\[
\underline R_{q,r}(a)
-
\min_{b\in\mathcal C_{q,r}}
\overline R_{q,r}(b)
\geq
\Delta_{q,a}
-
8\overline\lambda_q\epsilon_r.
\]
If
$
\epsilon_r
<
\frac{\Delta_{q,a}}{8\overline\lambda_q},
$
the right-hand side is strictly positive, and arm $a$ is eliminated. Since all suboptimal group gaps are positive and $\epsilon_r\to0$, every group eventually contains only $a_q^\star$ on $\mathcal E$. Combined with $\mathbb P_{\boldsymbol\mu}(\mathcal E)\geq1-\delta$ from Lemma~\ref{lem:confidence-event}, this proves the lemma.
\end{proof}

\subsection{Proof of Theorem~\ref{thm:gse-sample-complexity}}
\begin{proof}
We work on the event $\mathcal E$ throughout. Fix a group $q$. By Lemma~\ref{thm:gse-correctness}, every suboptimal representative candidate is eliminated once
\[
\epsilon_r<\frac{\Delta_{q,a}}{8\overline\lambda_q}.
\]
In particular, all suboptimal candidates are eliminated once $\epsilon_r<\rho_q$, so group $q$ is resolved no later than the first phase satisfying $\epsilon_r<\rho_q$.

Consider first the representative role of arm $a$ in group $q$. If $a\neq a_q^\star$, it is eliminated once
\[
\epsilon_r
<
\frac{\Delta_{q,a}}{8\overline\lambda_q}.
\]
If $a=a_q^\star$, it remains active until group $q$ is resolved. In both cases, this role has ended by the first phase satisfying
\[
\epsilon_r
<
\max\left\{
\rho_q,
\frac{\Delta_{q,a}}{8\overline\lambda_q}
\right\}.
\]

Next, consider the benchmark role of arm $a$ for an objective $m\in G_q$. If $g_{m,a}>4\epsilon_r$, then on $\mathcal E$,
\[
\widehat\mu_{a,m}^{(r)}
\leq
\mu_{a,m}+\epsilon_r
<
\mu_m^\star-3\epsilon_r
\leq
\widehat\mu_{m,r}^\star-2\epsilon_r,
\]
so $a$ is eliminated from $\mathcal B_{m,r+1}$. Otherwise, the benchmark role still ends once group $q$ is resolved, since $m$ then leaves $\mathcal M_r$. Therefore, this role has ended by the first phase satisfying
\[
\epsilon_r
<
\max\left\{
\rho_q,
\frac{g_{m,a}}{4}
\right\}.
\]

By Eq.~(\ref{eq:effective-arm-gap}), all roles of arm $a$ have ended by the phase
\[
r_a
=
\min\left\{
r\geq1:
\epsilon_r<\Gamma_a
\right\},
\]
so arm $a$ is pulled at most $n_{r_a}$ times.

By the minimality of $r_a$,
$
\epsilon_{r_a}
\geq
\frac{\Gamma_a}{2}
$
and
$
r_a
\leq
1+
\left\lceil
\log_2\frac{1}{\Gamma_a}
\right\rceil.
$
Using Eq.~(\ref{eq:phase-confidence}),
\[
\begin{aligned}
n_{r_a}
&\leq
1+
\frac{2\sigma^2}{\epsilon_{r_a}^2}
\log\left(
\frac{\pi^2KM r_a^2}{3\delta}
\right)\\
&\leq
1+
\frac{8\sigma^2}{\Gamma_a^2}
\log\left(
\frac{\pi^2KM}{3\delta}
\left(
1+
\left\lceil
\log_2\frac{1}{\Gamma_a}
\right\rceil
\right)^2
\right).
\end{aligned}
\]
Consequently, on $\mathcal E$, the total number of pulls is at most
\[
\sum_{a=1}^K
\left[
1+
\frac{8\sigma^2}{\Gamma_a^2}
\log\left(
\frac{\pi^2KM}{3\delta}
\left(
1+\left\lceil\log_2\frac{1}{\Gamma_a}\right\rceil
\right)^2
\right)
\right],
\]
which is of the order stated in Eq.~(\ref{eq:gse-explicit-complexity}). Since $\mathbb P_{\boldsymbol\mu}(\mathcal E)\geq1-\delta$ by Lemma~\ref{lem:confidence-event}, the high-probability guarantee follows.
\end{proof}

\section{Analysis of the Adaptive Method: \textsc{AFM-Elim}}
\subsection{Set-level confidence bounds and correctness}

Let $\mathcal E$ be the simultaneous confidence event defined in Eq.~(\ref{eq:good-event}). We first show that the loss bounds in Eq.~(\ref{eq:optimistic-pessimistic-set-loss}) are valid uniformly over all candidate sets. Importantly, this result does not require a union bound over the $\binom{K}{L}$ candidate sets.

\begin{lemma}[Set-level confidence bounds]\label{lem:set-level-confidence}
On the event $\mathcal E$, the following statements hold at every phase $r$.

\begin{enumerate}
    \item The best arm of every objective remains in its benchmark set:
    \[
    b_m^\star
    \in
    \mathcal B_{m,r},
    \qquad
    \forall m\in[M].
    \]

    \item For every $S\in\mathcal C_r$ and $m\in[M]$, the empirical objective-wise value satisfies
    \begin{equation}\label{eq:coverage-confidence-valid}
    \left|\widehat V_{m,r}(S)-V_m(S)\right|\leq\epsilon_r.
    \end{equation}

    \item For every $S\in\mathcal C_r$,
    \begin{equation}\label{eq:set-loss-confidence-valid}
    \underline\Phi_r(S)
    \leq
    \Phi(S)
    \leq
    \overline\Phi_r(S).
    \end{equation}

    \item The set-level loss bounds satisfy
    \begin{equation}\label{eq:set-loss-confidence-width}
    \Phi(S)-4\overline\lambda\epsilon_r
    \leq
    \underline\Phi_r(S)
    \leq
    \overline\Phi_r(S)
    \leq
    \Phi(S)+4\overline\lambda\epsilon_r.
    \end{equation}
\end{enumerate}
\end{lemma}

\begin{proof}
The benchmark-set statement follows from the same induction argument as in Lemma~\ref{lem:loss-confidence}, so $b_m^\star$ is never removed from $\mathcal B_{m,r}$. On $\mathcal E$, the nonexpansiveness of the maximum operator gives
\[
\left|\widehat V_{m,r}(S)-V_m(S)\right|
\leq
\max_{a\in S}\left|\widehat\mu_{a,m}^{(r)}-\mu_{a,m}\right|
\leq\epsilon_r,
\]
which proves Eq.~(\ref{eq:coverage-confidence-valid}). Because the best arm remains in $\mathcal B_{m,r}$, Eq.~(\ref{eq:empirical-ideal-concentration}) also holds. Therefore,
\[
\left|
\left(\widehat\mu_{m,r}^\star-\widehat V_{m,r}(S)\right)
-
\left(\mu_m^\star-V_m(S)\right)
\right|
\leq2\epsilon_r.
\]
The monotonicity and $1$-Lipschitz continuity of the positive-part operator in Eq.~(\ref{eq:optimistic-pessimistic-set-loss}) then yield Eq.~(\ref{eq:set-loss-confidence-valid}), as well as
\[
\underline\Phi_r(S)
\geq\Phi(S)-4\overline\lambda\epsilon_r,
\qquad
\overline\Phi_r(S)
\leq\Phi(S)+4\overline\lambda\epsilon_r,
\]
which proves Eq.~(\ref{eq:set-loss-confidence-width}).
\end{proof}

The set-level confidence event depends only on the $KM$ base-arm means. Consequently, the confidence radius scales logarithmically in $KM$ rather than in the exponentially large number of candidate sets.

\begin{lem}[Set-elimination correctness]\label{thm:afm-correctness}
Algorithm~\ref{alg:afm-elim} is $\delta$-PAC. More precisely, on the event $\mathcal E$:

\begin{enumerate}
    \item the optimal representative set $S_L^\star$ is never eliminated;
    \item every $S\neq S_L^\star$ that is still active is eliminated at the first phase satisfying
    \begin{equation}\label{eq:set-elimination-condition}
    \epsilon_r
    <
    \frac{\Delta_L(S)}{8\overline\lambda}.
    \end{equation}
\end{enumerate}
Therefore,
\[
\mathbb P_{\boldsymbol\mu}
\left(
\tau<\infty,
\widehat S_L=S_L^\star
\right)
\geq
1-\delta.
\]
\end{lem}

\begin{proof}
Let $B_r=\min_{S'\in\mathcal C_r}\overline\Phi_r(S')$ denote the elimination threshold at phase $r$. By Lemma~\ref{lem:set-level-confidence},
\[
\underline\Phi_r(S_L^\star)
\leq
\Phi(S_L^\star),
\]
and, for every $S\in\mathcal C_r$,
\[
\overline\Phi_r(S)
\geq
\Phi(S)
\geq
\Phi(S_L^\star).
\]
It follows that
\[
\underline\Phi_r(S_L^\star)
\leq
B_r,
\]
so $S_L^\star$ is never eliminated.

Now fix a suboptimal set $S\neq S_L^\star$. By Eq.~(\ref{eq:set-loss-confidence-width}),
\[
\underline\Phi_r(S)
\geq
\Phi(S)-4\overline\lambda\epsilon_r,
\]
and, since $S_L^\star\in\mathcal C_r$,
\[
B_r
\leq
\overline\Phi_r(S_L^\star)
\leq
\Phi(S_L^\star)+4\overline\lambda\epsilon_r.
\]
Therefore,
\[
\underline\Phi_r(S)-B_r
\geq
\Delta_L(S)-8\overline\lambda\epsilon_r.
\]
Whenever Eq.~(\ref{eq:set-elimination-condition}) holds, the right-hand side is strictly positive, and hence $S$ is removed from $\mathcal C_{r+1}$.

Since $\Delta_{L,\min}>0$ and $\epsilon_r\to0$, every suboptimal candidate set is eventually eliminated on $\mathcal E$. Combined with $\mathbb P_{\boldsymbol\mu}(\mathcal E)\geq1-\delta$ from Lemma~\ref{lem:confidence-event}, this proves the result.
\end{proof}

\subsection{Proof of Theorem~\ref{thm:afm-sample-complexity}}
\begin{proof}
We work on the event $\mathcal E$ throughout. Fix a base arm $a$, and consider first its role in the active candidate family.

If $a\notin S_L^\star$, every candidate set containing $a$ is suboptimal. By Lemma~\ref{thm:afm-correctness}, each such set $S$ is eliminated once
\[
\epsilon_r
<
\frac{\Delta_L(S)}{8\overline\lambda}.
\]
Therefore, all candidate sets containing $a$ have been eliminated once
\[
\epsilon_r
<
\min_{\substack{S\in\mathcal S_L\\a\in S}}
\frac{\Delta_L(S)}{8\overline\lambda}.
\]
If $a\in S_L^\star$, then $a$ remains in the union of active candidate sets until the algorithm terminates. Since every suboptimal set is eliminated once $\epsilon_r<\rho_L$, this role ends once
\[
\epsilon_r<\rho_L.
\]

Next, consider the benchmark role of arm $a$ for objective $m$. If
\[
g_{m,a}>4\epsilon_r,
\]
then, on $\mathcal E$,
\[
\widehat\mu_{a,m}^{(r)}
\leq
\mu_{a,m}+\epsilon_r
<
\mu_m^\star-3\epsilon_r
\leq
\widehat\mu_{m,r}^\star-2\epsilon_r,
\]
so $a$ is removed from $\mathcal B_{m,r+1}$. Otherwise, this role ends when the algorithm terminates. Hence, the benchmark role has ended by the first phase satisfying
\[
\epsilon_r
<
\max\left\{
\rho_L,
\frac{g_{m,a}}{4}
\right\}.
\]

By Eq.~(\ref{eq:latent-effective-gap}), all roles of arm $a$ have ended by the phase
\[
r_a
=
\min\left\{
r\geq1:
\epsilon_r<\Gamma_{L,a}
\right\},
\]
so the arm is pulled at most $n_{r_a}$ times. By the minimality of $r_a$,
\[
\epsilon_{r_a}
\geq
\frac{\Gamma_{L,a}}{2},
\qquad
r_a
\leq
1+
\left\lceil
\log_2\frac{1}{\Gamma_{L,a}}
\right\rceil.
\]
Substituting these inequalities into the definition of $n_r$ gives
\[
n_{r_a}
\leq
1+
\frac{8\sigma^2}{\Gamma_{L,a}^2}
\log\left(
\frac{\pi^2KM}{3\delta}
\left(
1+
\left\lceil
\log_2\frac{1}{\Gamma_{L,a}}
\right\rceil
\right)^2
\right).
\]
Consequently, on $\mathcal E$, the total number of pulls is at most
\[
\sum_{a=1}^K
\left[
1+
\frac{8\sigma^2}{\Gamma_{L,a}^2}
\log\left(
\frac{\pi^2KM}{3\delta}
\left(
1+\left\lceil\log_2\frac{1}{\Gamma_{L,a}}\right\rceil
\right)^2
\right)
\right],
\]
which is of the order stated in Eq.~(\ref{eq:afm-explicit-sample-complexity}). Since $\mathbb P_{\boldsymbol\mu}(\mathcal E)\geq1-\delta$ by Lemma~\ref{lem:confidence-event}, the high-probability guarantee follows.
\end{proof}

\subsection{Recovery of the latent objective groups}\label{app:latent-group-recovery}

Identifying $S_L^\star$ does not by itself determine a unique objective partition. To recover a unique assignment, we impose the following condition.

\begin{assumption}\label{ass:unique-latent-assignment}
For every objective $m\in[M]$, there exists a unique arm
$\pi_L^\star(m)\in S_L^\star$ satisfying
\[
\pi_L^\star(m)
=
\arg\max_{a\in S_L^\star}\mu_{a,m}.
\]
\end{assumption}

The latent group represented by $a\in S_L^\star$ is then
\begin{equation}
G_a^\star
=
\left\{
m\in[M]:
\pi_L^\star(m)=a
\right\}.
\end{equation}

When $L=1$, the assignment is trivial and requires no additional samples. We therefore assume $L\geq2$ below and define the assignment gap
\begin{equation}
\zeta_m
=
\mu_{\pi_L^\star(m),m}
-
\max_{a\in S_L^\star\setminus\{\pi_L^\star(m)\}}
\mu_{a,m},
\end{equation}
together with
\[
\zeta_{\min}
=
\min_{m\in[M]}\zeta_m.
\]

After $S_L^\star$ has been identified, the learner can continue sampling its $L$ arms until, for every objective $m$, there exists an arm $\widehat\pi(m)\in S_L^\star$ such that
\begin{equation}\label{eq:assignment-stopping}
\widehat\mu_{\widehat\pi(m),m}^{(r)}
-
\max_{a\in S_L^\star\setminus\{\widehat\pi(m)\}}
\widehat\mu_{a,m}^{(r)}
>
2\epsilon_r.
\end{equation}
The output groups are
\[
\widehat G_a
=
\left\{
m\in[M]:
\widehat\pi(m)=a
\right\},
\qquad
a\in S_L^\star.
\]

\begin{corollary}[Latent-group recovery]
Suppose that $L\geq2$ and that Assumption~\ref{ass:unique-latent-assignment} holds. On the event $\mathcal E$, condition Eq.~(\ref{eq:assignment-stopping}) is satisfied for all objectives once
\[
\epsilon_r
<
\frac{\zeta_{\min}}{4},
\]
and the resulting assignments and groups satisfy
\[
\widehat\pi(m)=\pi_L^\star(m),
\qquad
\widehat G_a=G_a^\star.
\]
The additional sample complexity is at most
\begin{equation}\label{eq:assignment-sample-complexity}
\widetilde O\left(
\frac{\sigma^2L}{\zeta_{\min}^2}
\log\frac{KM}{\delta}
\right).
\end{equation}
\end{corollary}

\begin{proof}
For every objective $m$, on $\mathcal E$,
\[
\widehat\mu_{\pi_L^\star(m),m}^{(r)}
-
\max_{a\in S_L^\star\setminus\{\pi_L^\star(m)\}}
\widehat\mu_{a,m}^{(r)}
\geq
\zeta_m-2\epsilon_r.
\]
When $\epsilon_r<\zeta_{\min}/4$, the right-hand side exceeds $2\epsilon_r$, so Eq.~(\ref{eq:assignment-stopping}) holds with $\widehat\pi(m)=\pi_L^\star(m)$. Conversely, on $\mathcal E$, any arm $\widehat\pi(m)$ satisfying Eq.~(\ref{eq:assignment-stopping}) has a true mean strictly larger than those of all other arms in $S_L^\star$, so the certified assignment is correct. At the first phase with $\epsilon_r<\zeta_{\min}/4$, the same calculation as in the proof of Theorem~\ref{thm:gse-sample-complexity} gives $n_r=\widetilde O(\sigma^2\zeta_{\min}^{-2}\log(KM/\delta))$. Sampling the $L$ selected arms to this precision proves Eq.~(\ref{eq:assignment-sample-complexity}).
\end{proof}

\section{Proof of Theorem~\ref{thm:latent-lower-bound}}
\begin{proof}
If $\mathbb E_{\boldsymbol\mu}[\tau]=+\infty$, the lower bound holds trivially. We therefore assume $\mathbb E_{\boldsymbol\mu}[\tau]<+\infty$.

Let
\[
E
=
\left\{
\widehat S_L=S_L^\star(\boldsymbol\mu)
\right\}.
\]
Since the algorithm is $\delta$-PAC,
\[
\mathbb P_{\boldsymbol\mu}(E)\geq1-\delta.
\]
For every $\boldsymbol\nu\in\operatorname{Alt}_L(\boldsymbol\mu)$, the set $S_L^\star(\boldsymbol\mu)$ is not optimal under $\boldsymbol\nu$, and hence
\[
\mathbb P_{\boldsymbol\nu}(E)\leq\delta.
\]

The change-of-measure inequality for adaptive sampling gives
\[
\sum_{a=1}^K
\mathbb E_{\boldsymbol\mu}[N_a(\tau)]
d(\boldsymbol\mu_a,\boldsymbol\nu_a)
\geq
\operatorname{KL}
\left(
\mathbb P_{\boldsymbol\mu}(E),
\mathbb P_{\boldsymbol\nu}(E)
\right)
\geq
\operatorname{KL}(1-\delta,\delta),
\]
where the last step uses the monotonicity of $\operatorname{KL}$ and $\delta<1/2$. Let
\[
w_a
=
\frac{
\mathbb E_{\boldsymbol\mu}[N_a(\tau)]
}{
\mathbb E_{\boldsymbol\mu}[\tau]
}.
\]
Then $\boldsymbol w\in\Delta_K$, and since the previous inequality holds for every alternative,
\[
\mathbb E_{\boldsymbol\mu}[\tau]
\inf_{\boldsymbol\nu\in\operatorname{Alt}_L(\boldsymbol\mu)}
\sum_{a=1}^K
w_a
d(\boldsymbol\mu_a,\boldsymbol\nu_a)
\geq
\operatorname{KL}(1-\delta,\delta).
\]
By the definition of $T_L^\star(\boldsymbol\mu)$,
\[
\inf_{\boldsymbol\nu\in\operatorname{Alt}_L(\boldsymbol\mu)}
\sum_{a=1}^K
w_a
d(\boldsymbol\mu_a,\boldsymbol\nu_a)
\leq
\left(
T_L^\star(\boldsymbol\mu)
\right)^{-1}.
\]
Combining the last two inequalities proves Eq.~(\ref{eq:latent-lower-bound}). Finally, since
\[
\operatorname{KL}(1-\delta,\delta)
\sim
\log\frac{1}{\delta}
\]
as $\delta\downarrow0$, Eq.~(\ref{eq:latent-asymptotic-lower-bound}) follows.
\end{proof}

\section{Proof of Theorem~\ref{thm:biafm-greedy}}

\begin{proof}
Set $\eta=\epsilon/(4\overline\lambda)$ and define
\[
\mathcal E_{\rm Bi}
=
\left\{
|\widehat\mu_{a,m}-\mu_{a,m}|\leq\eta,
\ \forall a\in[K],m\in[M]
\right\}.
\]
For each arm--objective pair, the sub-Gaussian concentration inequality gives
\[
\mathbb P_{\boldsymbol\mu}\left(|\widehat\mu_{a,m}-\mu_{a,m}|>\eta\right)
\leq2\exp\left(-\frac{n\eta^2}{2\sigma^2}\right).
\]
The prescribed value of $n$ and $\eta^2=\epsilon^2/(16\overline\lambda^2)$ imply
\[
\frac{n\eta^2}{2\sigma^2}
\geq
\log\left(\frac{2KM}{\delta}\right).
\]
Hence, each failure probability is at most $\delta/(KM)$, and a union bound yields $\mathbb P_{\boldsymbol\mu}(\mathcal E_{\rm Bi})\geq1-\delta$.

On $\mathcal E_{\rm Bi}$, the nonexpansiveness of the maximum operator gives, for every $m$,
\[
|\widehat\mu_m^\star-\mu_m^\star|
=
\left|\max_a\widehat\mu_{a,m}-\max_a\mu_{a,m}\right|
\leq\max_a|\widehat\mu_{a,m}-\mu_{a,m}|
\leq\eta.
\]
Consequently, for every $a$ and $m$,
\begin{equation}\label{eq:biafm-deficit-concentration}
|\widehat d_{a,m}-d_{a,m}|
\leq
\lambda_m\left(|\widehat\mu_m^\star-\mu_m^\star|+|\widehat\mu_{a,m}-\mu_{a,m}|\right)
\leq2\overline\lambda\eta
=
\frac{\epsilon}{2}.
\end{equation}
Define the empirical loss
\begin{equation}\label{eq:empirical-weighted-loss}
\widehat\Phi(S)=\max_{m\in[M]}\min_{a\in S}\widehat d_{a,m}.
\end{equation}
Since both the minimum and the maximum operators are nonexpansive in the sup norm, Eq.~(\ref{eq:biafm-deficit-concentration}) implies, simultaneously for every nonempty $S\subseteq[K]$,
\begin{equation}\label{eq:biafm-uniform-loss-concentration}
|\widehat\Phi(S)-\Phi(S)|\leq\frac{\epsilon}{2}.
\end{equation}

Let $\gamma_{\rm emp}^\star=\widehat\Phi(S_L^\star)$. For each objective $m$, some arm in $S_L^\star$ has an empirical deficit of at most $\gamma_{\rm emp}^\star$, so the $L$ arms in $S_L^\star$ cover $[M]$ at this threshold. Thus,
\begin{equation}\label{eq:empirical-threshold-cover-size}
\operatorname{OPT}_{\rm SC}(\gamma_{\rm emp}^\star)\leq L,
\qquad
|S_{\gamma_{\rm emp}^\star}|\leq H_ML.
\end{equation}
Moreover, because both extrema in Eq.~(\ref{eq:empirical-weighted-loss}) are taken over finite sets, there exist $m^\prime\in[M]$ and $a^\prime\in S_L^\star$ such that $\gamma_{\rm emp}^\star=\widehat d_{a^\prime,m^\prime}$; hence $\gamma_{\rm emp}^\star\in\widehat{\mathcal T}$. Algorithm~\ref{alg:biafm-greedy} scans the thresholds in increasing order, and Eq.~(\ref{eq:empirical-threshold-cover-size}), together with the fact that a greedy cover uses at most $K$ distinct arms, shows that the acceptance test is passed no later than at $\gamma_{\rm emp}^\star$. The returned threshold $\widehat\gamma$ therefore satisfies $\widehat\gamma\leq\gamma_{\rm emp}^\star$.

Because $\widehat S$ covers every objective at threshold $\widehat\gamma$, we have $\widehat\Phi(\widehat S)\leq\widehat\gamma$. Combining this with Eq.~(\ref{eq:biafm-uniform-loss-concentration}) gives
\[
\begin{aligned}
\Phi(\widehat S)
&\leq\widehat\Phi(\widehat S)+\frac{\epsilon}{2}
\leq\widehat\gamma+\frac{\epsilon}{2}
\leq\widehat\Phi(S_L^\star)+\frac{\epsilon}{2}
\leq\Phi(S_L^\star)+\epsilon.
\end{aligned}
\]
The budget bound follows directly from the acceptance test, and the preceding argument holds on an event of probability at least $1-\delta$.

The algorithm uses exactly $Kn$ arm pulls, which gives Eq.~(\ref{eq:biafm-sample-complexity}). There are at most $KM$ distinct thresholds. Under a direct implementation, computing the empirical means costs $O(KMn)$ scalar operations, sorting the thresholds costs $O(KM\log(KM))$, and constructing all coverage sets and running the greedy cover at every threshold costs at most $O(K^3M^2)$. Hence, the overall computation is polynomial in the stated parameters.
\end{proof}

\section{Proof of Theorem~\ref{thm:afm-asymptotic-optimality}}
\begin{proof}
The $\delta$-PAC property follows directly from Eq.~(\ref{eq:valid-glr-threshold}): an incorrect recommendation can occur only if, for some $t$,
\[
\widehat S_t
\neq
S_L^\star(\boldsymbol\mu)
\quad\text{and}\quad
Z_L(t)\geq\beta(t,\delta),
\]
and this event has probability at most $\delta$.

Forced exploration guarantees that
\[
N_a(t)\longrightarrow\infty,
\qquad
\forall a\in[K].
\]
Therefore, by the strong law of large numbers,
\[
\widehat{\boldsymbol\mu}_a(t)
\longrightarrow
\boldsymbol\mu_a
\quad\text{almost surely}.
\]
Since $\Delta_{L,\min}>0$, the empirical optimal set eventually coincides with the true optimal set, i.e.,
\[
\widehat S_t=S_L^\star
\]
for all sufficiently large $t$, almost surely.

For $\boldsymbol w\in\Delta_K$ and an instance $\boldsymbol\theta$, define the information rate
\[
I_L(\boldsymbol w,\boldsymbol\theta)
=
\inf_{\boldsymbol\nu\in\operatorname{Alt}_L(\boldsymbol\theta)}
\sum_{a=1}^K w_a d(\boldsymbol\theta_a,\boldsymbol\nu_a).
\]
Let $\boldsymbol w(t)=\mathbf N(t)/t$. Once the empirical answer has stabilized, Eq.~(\ref{eq:latent-glr-statistic}) gives
\[
\frac{Z_L(t)}{t}
=
I_L\bigl(\boldsymbol w(t),\widehat{\boldsymbol\mu}(t)\bigr).
\]
The exact allocation oracle and the tracking step of Algorithm~\ref{alg:afm-tas} constitute precisely the Track-and-Stop tracking rule applied to the latent answer space. The standard tracking argument therefore implies that, along the realized trajectory, the empirical sampling proportions approach the optimal allocation set of the true instance. Consequently, we may choose $\boldsymbol w_t^\star\in\mathcal W_L^\star(\boldsymbol\mu)$ such that $\|\boldsymbol w(t)-\boldsymbol w_t^\star\|\to0$. Together with the strong law and the standard continuity argument for Track-and-Stop statistics, this yields
\[
\left|
I_L\bigl(\boldsymbol w(t),\widehat{\boldsymbol\mu}(t)\bigr)
-
I_L(\boldsymbol w_t^\star,\boldsymbol\mu)
\right|
\longrightarrow
0
\quad\text{almost surely}.
\]
Since every $\boldsymbol w_t^\star$ maximizes the information rate, $I_L(\boldsymbol w_t^\star,\boldsymbol\mu)=1/T_L^\star(\boldsymbol\mu)$. Therefore,
\[
\frac{Z_L(t)}{t}
\longrightarrow
\frac{1}{T_L^\star(\boldsymbol\mu)}
\quad\text{almost surely}.
\]
Hence, for every $\eta\in(0,1)$, there exists an almost surely finite random time $t_\eta$ such that
\[
Z_L(t)
\geq
\frac{(1-\eta)t}{T_L^\star(\boldsymbol\mu)}
\]
for all $t\geq t_\eta$.

Let
\[
t_\delta
=
\left\lceil
\frac{1+2\eta}{1-\eta}
T_L^\star(\boldsymbol\mu)
\log\frac{1}{\delta}
\right\rceil.
\]
By Eq.~(\ref{eq:threshold-asymptotics}),
\[
\beta(t_\delta,\delta)
\leq
(1+\eta)
\log\frac{1}{\delta}
\]
for all sufficiently small $\delta$. Therefore, on the event $\{t_\eta\leq t_\delta\}$,
\[
Z_L(t_\delta)
\geq
(1+2\eta)
\log\frac{1}{\delta}
>
\beta(t_\delta,\delta)
\]
for all sufficiently small $\delta$. Since $t_\delta\to\infty$ and $t_\eta<\infty$ almost surely, $\mathbb P_{\boldsymbol\mu}(t_\eta\leq t_\delta)\to1$. It follows that
\[
\mathbb P_{\boldsymbol\mu}(\tau_\delta\leq t_\delta)
\longrightarrow1.
\]
Since $(1+2\eta)/(1-\eta)$ can be made arbitrarily close to one by choosing $\eta$ sufficiently small, relabeling the constant proves Eq.~(\ref{eq:afm-asymptotic-upper}).
\end{proof}

\clearpage

\bibliographystyle{plainnat}
\bibliography{ref}
\end{document}